\documentclass{article}

\usepackage{microtype}
\usepackage{graphicx}
\usepackage{caption}
\usepackage{subcaption}
\usepackage{booktabs} 
\usepackage{slashbox}

\usepackage{hyperref}

\usepackage[accepted]{icml2018}

\usepackage{natbib}
\usepackage{amsmath}
\usepackage{amsthm}
\usepackage{amssymb}
\usepackage{mathtools}
\usepackage{tikz}
\usepackage{xcolor}
\usetikzlibrary{arrows}

\usepackage{algorithm}
\usepackage{algorithmic}
\usepackage{hyperref}
\usepackage{bm}

\newcommand{\poly}{\mathrm{poly}}

\newcommand{\unif}{\mathrm{unif}}

\newcommand{\firstlayer}{w}
\newcommand{\firstlayerWN}{v}
\newcommand{\secondlayer}{a}

\def\vw{\mathbf{w}}
\def\va{\mathbf{a}}
\def\vv{\mathbf{v}}
\def\vZ{\mathbf{Z}}

\newcommand{\mat}[1]{\mathbf{#1}}
\newcommand{\vect}[1]{\mathbf{#1}}
\newcommand{\norm}[1]{\left\|#1\right\|}
\newcommand{\abs}[1]{\left|#1\right|}
\newcommand{\expect}{\mathbb{E}}
\newcommand{\prob}{\mathbb{P}}

\newcommand{\diff}{\text{d}}

\newcommand{\indict}{\mathbb{I}}

\newcommand{\relu}[1]{\sigma\left(#1\right)}

\newcommand{\ab}{\mathbf{B}}

\newtheorem{thm}{Theorem}[section]
\newtheorem{lem}{Lemma}[section]
\newtheorem{cor}{Corollary}[section]

\icmltitlerunning{Gradient Descent Learns One-hidden-layer CNN}

\begin{document}

\twocolumn[
\icmltitle{Gradient Descent Learns One-hidden-layer CNN: \\
	Don't be Afraid of Spurious Local Minima} 



\icmlsetsymbol{equal}{*}

\begin{icmlauthorlist}
\icmlauthor{Simon S. Du}{ml}
\icmlauthor{Jason D. Lee}{usc}
\icmlauthor{Yuandong Tian}{fb}
\icmlauthor{Barnab\'{a}s P\'{o}czos}{ml}
\icmlauthor{Aarti Singh}{ml}
\end{icmlauthorlist}

\icmlaffiliation{ml}{Machine Learning Department, Carnegie Mellon University}
\icmlaffiliation{usc}{Department of Data Sciences and Operations, University of Southern California}
\icmlaffiliation{fb}{Facebook Artificial Intelligence Research}

\icmlcorrespondingauthor{Simon S. Du}{ssdu@cs.cmu.edu}

\icmlkeywords{non-convex, convolutional, neural network,}

\vskip 0.3in
]



\printAffiliationsAndNotice{}  

\begin{abstract}
	\label{sec:abs}
	We consider the problem of learning a one-hidden-layer neural network with non-overlapping convolutional layer and ReLU activation, i.e., $f(\mat{Z}, \vw, \va) = \sum_j a_j\sigma(\vw^T\vZ_j)$, in which both the convolutional weights $\vw$ and the output weights $\va$ are parameters to be learned. When the labels are the outputs from a teacher network of the same architecture with fixed weights $(\vw^*, \va^*)$, we prove that with Gaussian input $\vZ$, there is a spurious local minimizer. 
Surprisingly, in the presence of the spurious local minimizer, gradient descent with weight normalization from randomly initialized weights can still be proven to recover the true parameters with constant probability, which can be boosted to probability $1$ with multiple restarts. 
We also show that with constant probability, the same procedure could also converge to the spurious local minimum, showing that the local minimum plays a non-trivial role in the dynamics of gradient descent. 
Furthermore, a quantitative analysis shows that the gradient descent dynamics has two phases: it starts off slow, but converges much faster after several iterations.

\end{abstract}

\section{Introduction}
\label{sec:intro}
Deep convolutional neural networks (DCNN) have achieved the state-of-the-art performance in many applications such as computer vision~\citep{krizhevsky2012imagenet}, natural language processing~\citep{dauphin2016language} and reinforcement learning applied in classic games like Go~\citep{silver2016mastering}.
Despite the highly non-convex nature of the objective function, simple first-order algorithms like stochastic gradient descent and its variants often train such networks successfully.
Why such simple methods in learning DCNN is successful remains elusive from the optimization perspective.

Recently, a line of research~\citep{tian2017analytical,brutzkus2017globally,li2017convergence,soltanolkotabi2017learning,shalev2017weight} assumed the input distribution is Gaussian and showed that stochastic gradient descent with random or $\vect{0}$ initialization is able to train a neural network $f(\mat{Z}, \{\vw_j\}) = \sum_j a_j\sigma(\vw_j^T\vZ)$ with ReLU activation $\sigma(x) = \max(x, 0)$ in polynomial time.
However, these results all assume there is only one unknown layer $\{\vw_j\}$, while $\va$ is a fixed vector. A natural question thus arises:
\begin{center}
\emph{Does randomly initialized (stochastic) gradient descent learn  neural networks with multiple layers?}
\end{center}

\begin{figure*}
	\centering
	\begin{subfigure}[t]{0.45\textwidth}
		\includegraphics[width=\textwidth]{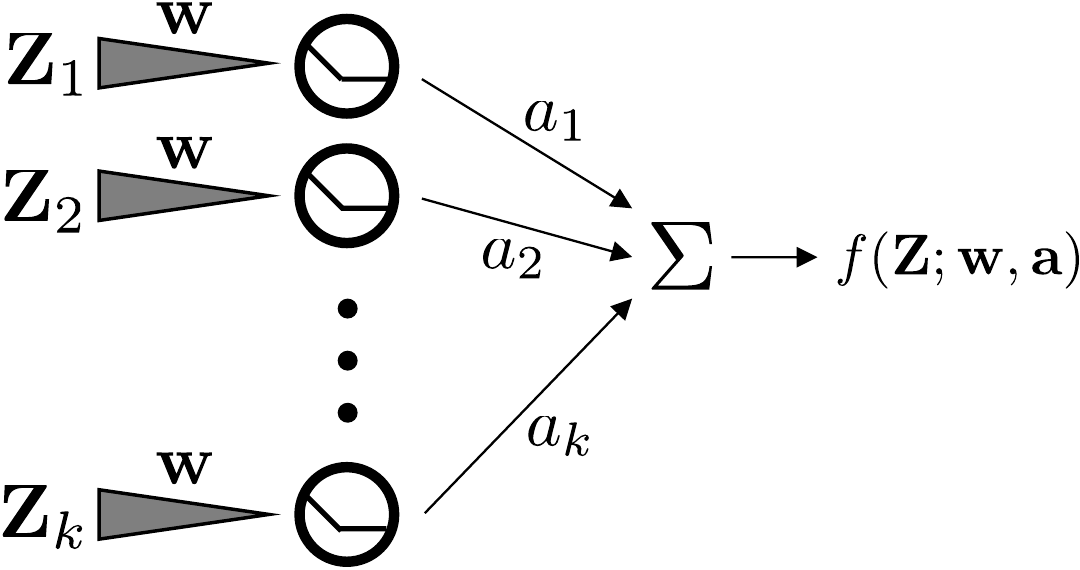}
		\caption{Convolutional neural network with an unknown non-overlapping filter and an unknown output layer.
	In the first (hidden) layer, a filter $\vect{\firstlayer}$ is applied to nonoverlapping parts of the input $\vect{x}$, which then passes through a ReLU activation function. 
	The final output is the inner product between an output weight vector $\vect{\secondlayer}$ and the hidden layer outputs.
	}\label{fig:architecture}
	\end{subfigure}	
	\qquad	
	\begin{subfigure}[t]{0.45\textwidth}
	\includegraphics[width=0.9\textwidth]{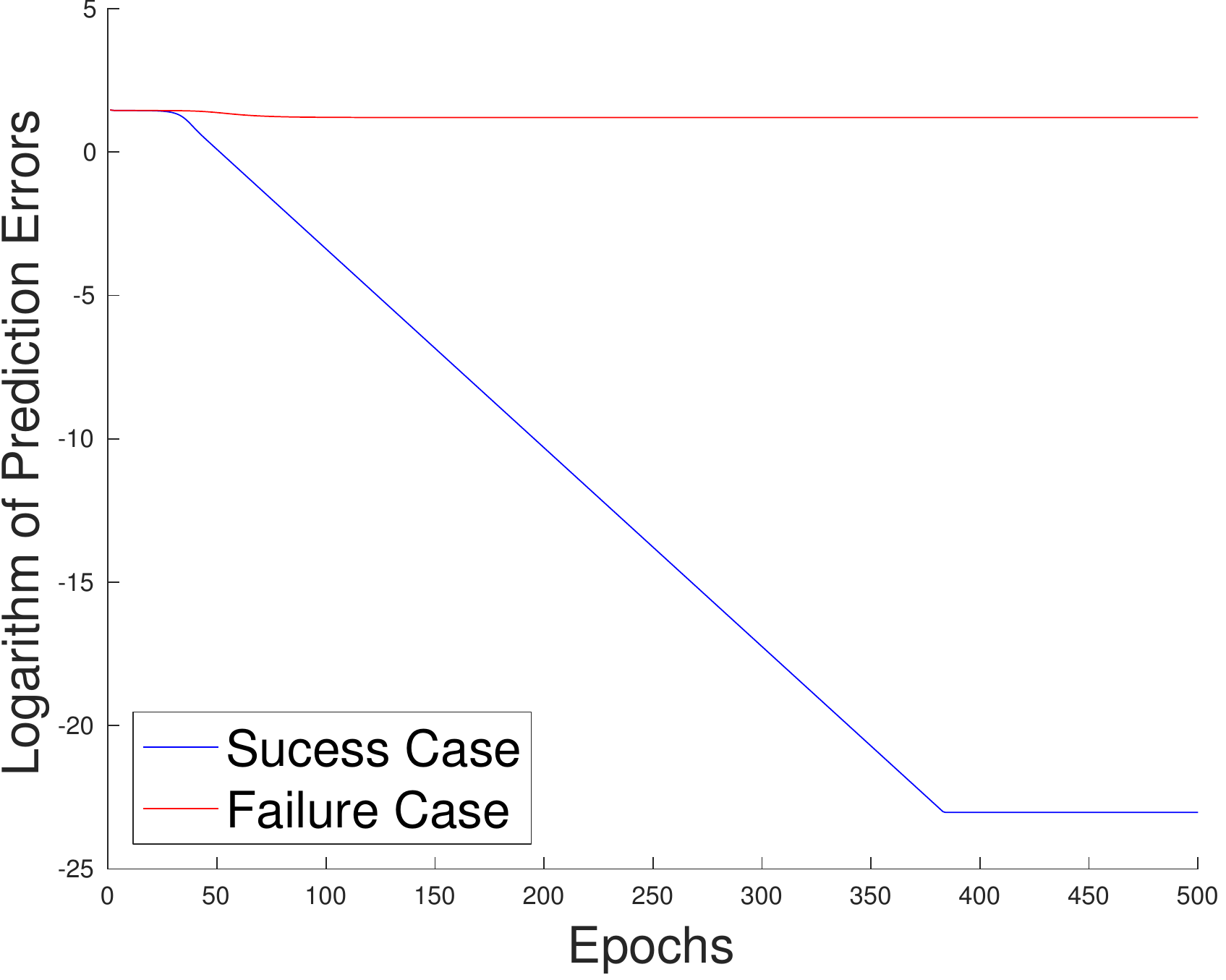}
		\caption{The convergence of gradient descent for learning a CNN described in Figure~\ref{fig:architecture} with Gaussian input using different initializations.
The success case and the failure case correspond to convergence to the global minimum and the spurious local minimum, respectively.
In the first $\sim 50$ iterations the convergence is slow. 
After that gradient descent converges at a fast linear rate. 
			}\label{fig:gd_convergence}
	\end{subfigure}	
	\caption{Network architecture that we consider in this paper and convergence of gradient descent for learning the parameters of this network.
} 
    \label{fig:architecture_and_convergence}
\end{figure*}

In this paper, we take an important step by showing that randomly initialized gradient descent learns a non-linear convolutional neural network with \emph{two} unknown layers $\vw$ and $\va$. 
To our knowledge, our work is the first of its kind. 

Formally, we consider the convolutional case in which a filter $\vw$ is shared among different hidden nodes. Let $\vect{x} \in \mathbb{R}^{d}$ be an input sample, e.g., an image. We generate $k$ patches from $\vect{x}$, each with size $p$: $\mat Z \in \mathbb{R}^{p \times k}$ where the $i$-th column is the $i$-th patch generated by selecting some coordinates of $\vect{x}$: $\vect{Z}_i = \vect{Z}_i(\vect{x})$. 
We further assume there is no overlap between patches.
Thus, the neural network function has the following form:
\begin{align*}
f(\mat{Z},\vect{w},\vect{\secondlayer}) =  \sum_{i=1}^{k} \secondlayer_i\sigma\left(\vect{w}^\top \vect{Z}_i\right). 
\end{align*}

We focus on the realizable case, i.e., the label is generated according to $y = f\left(\mat{Z},\vect{\firstlayer}^*,\vect{\secondlayer}^*\right)$ for some true parameters $\vect{\firstlayer}^*$ and $\vect{\secondlayer}^*$ and use $\ell_2$ loss to learn the parameters:\begin{align*}
\min_{\vect{w},\vect{\secondlayer}} \ell(\vect{Z},\vect{w},\vect{\secondlayer}) := \frac{1}{2}\left(f\left(\mat{Z},\vect{w},\vect{\secondlayer}\right) - f\left(\mat{Z},\vect{\firstlayer}^*,\vect{\secondlayer}^*\right)\right)^2. \label{eqn:individual_obj}
\end{align*} 
We assume $\vect{x}$ is sampled from a Gaussian distribution and there is no overlap between patches.
This assumption is equivalent to that each entry of $\mat{Z}$ is sampled from a Gaussian distribution~\citep{brutzkus2017globally,zhong2017recovery}.
Following~\citep{zhong2017learning,zhong2017recovery,li2017convergence,tian2017analytical,brutzkus2017globally,shalev2017weight}, 
in this paper, we mainly focus on the population loss:
\begin{align*}
	\ell\left(\vect{w},\vect{\secondlayer}\right) := \frac12\expect_{\mat{Z}}\left[\left(f\left(\mat{Z},\vect{\firstlayer},\vect{\secondlayer}\right)-f\left(\mat{Z},\vect{\firstlayer}^*,\vect{\secondlayer}^*\right)\right)^2\right].
\end{align*}
We study whether the global convergence $\vw\rightarrow\vw^*$ and $\va\rightarrow\va^*$ can be achieved when optimizing $\ell(\vect{w},\vect{a})$ using randomly initialized gradient descent.

A crucial difference between our two-layer network and previous one-layer models is there is a positive-homogeneity issue.
That is, for any $c > 0$, $f\left(\mat{Z},c\vect{w},\frac{\vect{\secondlayer}}{c}\right) = f\left(\mat{Z},\vect{w},\vect{\secondlayer}\right)$.
This interesting
property allows the network to be rescaled without changing the function computed by the network.
As reported by~\cite{neyshabur2015path}, it is desirable to have scaling-invariant learning algorithm to stabilize the training process.

One commonly used technique to achieve stability is \emph{weight-normalization} introduced by~\citet{salimans2016weight}. 
As reported in~\cite{salimans2016weight}, this re-parametrization improves the conditioning of the gradient  because it couples the magnitude of the weight vector from the direction of the weight vector and empirically accelerates  stochastic gradient descent optimization.

In our setting, we re-parametrize the first layer as $\vect{w} = \frac{\vect{\firstlayerWN}}{\norm{\vect{\firstlayerWN}}_2}$ and the prediction function becomes
\begin{align}
f\left(\mat{Z},\vect{\firstlayerWN},\vect{\secondlayer}\right) = \sum_{i=1}^{k} {\secondlayer}_i \frac{\relu{\mat{Z}_i^\top \vect{\firstlayerWN}}}{\norm{\vect{\firstlayerWN}}_2}. 
\end{align}
The loss function is
\begin{align}
\ell\left(\vect{\vect{\firstlayerWN},\vect{\secondlayer}}\right) = \frac{1}{2}\expect_{\mat{Z}}\left[\left(f\left(\mat{Z},\vect{\firstlayerWN},\vect{\secondlayer}\right) - f\left(\mat{Z},\vect{\firstlayerWN}^*,\vect{\secondlayer}^*\right)\right)^2\right]. \label{eqn:pop_obj_WN}
\end{align}
In this paper we focus on using randomly initialized gradient descent for learning this convolutional neural network.
The pseudo-code is listed in Algorithm~\ref{algo:weight_normalization_gd}.\footnote{With some simple calculations, we can see the optimal solution for $\vect{\secondlayer}$ is unique, which we denote as $\vect{\secondlayer}^*$ whereas the optimal for $\vect{\firstlayerWN}$ is not because for every optimal solution $\vect{\firstlayerWN}^*$, $c\vect{\firstlayerWN}^*$ for $c>0$ is also an optimal solution.
In this paper, with a little abuse of the notation, we use $\vect{\firstlayerWN}^*$ to denote the equivalent class of optimal solutions.
}

\textbf{Main Contributions.} Our paper have three contributions. First, we show if $(\vect{\firstlayerWN},\vect{\secondlayer})$ is initialized by a specific \emph{random initialization}, then with high probability, 
gradient descent from $(\vect{\firstlayerWN},\vect{\secondlayer})$ converges to teacher's parameters $(\vv^*, \va^*)$. We can further boost the success rate with more trials.

Second, perhaps surprisingly, we prove that the objective function (Equation~\eqref{eqn:pop_obj_WN}) \emph{does} have a spurious local minimum: using the same random initialization scheme, there exists a pair $(\vect{\tilde \firstlayerWN}^{0},\vect{\tilde \secondlayer}^{0}) \in S_\pm(\vv, \va)$ so that gradient descent from $(\vect{\tilde\firstlayerWN}^0,\vect{\tilde\secondlayer}^0)$ converges to this bad local minimum.
In contrast to previous works on guarantees for non-convex objective functions whose landscape satisfies ``no spurious local minima'' property~\citep{li2016symmetry,ge2017no,ge2016matrix,bhojanapalli2016global,ge2017learning,kawaguchi2016deep}, our result provides a concrete counter-example and highlights a conceptually surprising phenomenon: 
\begin{center}
\emph{Randomly initialized local search can find a global minimum in the presence of spurious local minima.}
\end{center}

Finally, we conduct a quantitative study of the dynamics of gradient descent.
We show that the dynamics of Algorithm~\ref{algo:weight_normalization_gd} has two phases. 
At the beginning (around first 50 iterations in Figure~\ref{fig:gd_convergence}), because the magnitude of initial signal (angle between $\vect{\firstlayerWN}$ and $\vect{\firstlayer}^*$) is small, the prediction error drops slowly. 
After that, when the signal becomes stronger, gradient descent converges at a much faster rate and the prediction error drops quickly.

\begin{algorithm}[tb]
	\caption{Gradient Descent for Learning One-Hidden-Layer CNN with Weight Normalization}
    \label{algo:weight_normalization_gd}
	\begin{algorithmic}[1]
		\STATE \textbf{Input}: Initialization $\vect{v}_0\in\mathbb{R}^p$, $\vect{\secondlayer}_0 \in \mathbb{R}^{k}$, learning rate $\eta$.
		\FOR{$t = 1,2,\ldots$}
		\STATE $\vect{\firstlayerWN}^{t+1}\leftarrow \vect{\firstlayerWN}^t - \eta\frac{\partial \ell\left(\vect{\firstlayerWN}^t,\vect{\secondlayer}^t\right)}{\partial \vect{\firstlayerWN}^t}$,
		\STATE $\vect{\secondlayer}^{t+1} \leftarrow \vect{\secondlayer}^t - \eta\frac{\partial \ell\left(\vect{\firstlayerWN}^{t},\vect{\secondlayer}^t\right)}{\partial \vect{\secondlayer}^t}$.
		\ENDFOR     
	\end{algorithmic}
\end{algorithm}

\textbf{Technical Insights.}
\label{sec:technique}
The main difficulty of analyzing the convergence is the presence of local minima.
Note that local minimum and the global minimum are disjoint (c.f. Figure~\ref{fig:gd_convergence}).
The key technique we adopt is to characterize the attraction basin for each minimum.
We consider the sequence $\left\{\left(\vect{\firstlayerWN}^t,\vect{\secondlayer}^t\right)\right\}_{t=0}^{\infty}$ generated by Algorithm~\ref{algo:weight_normalization_gd} with step size $\eta$ using initialization point $\left(\vect{\firstlayerWN}^0,\vect{\secondlayer}^0\right)$.
The attraction basin for a minimum $\left(\vect{\firstlayerWN}^*,\vect{\secondlayer}^*\right)$ is defined as the \begin{align*}
	\ab\left(\vect{\firstlayerWN}^*,\vect{\secondlayer}^*\right) = \left\{\left(\vect{\firstlayerWN}^0,\vect{\secondlayer}^0\right), \lim\limits_{t\rightarrow \infty}\left(\vect{\firstlayerWN}^t,\vect{\secondlayer}^t\right) \rightarrow \left(\vect{\firstlayerWN}^*,\vect{\secondlayer}^*\right)\right\}
\end{align*}
The goal is to find a distribution $\mathcal{G}$ for weight initialization so that the probability that the initial weights are in $\ab\left(\vect{\firstlayerWN}^*,\vect{\secondlayer}^*\right)$ of the global minimum is bounded below: 
\begin{align*}
	\prob_{\left(\vect{\firstlayerWN}^0, \vect{\secondlayer}^0\right)\sim \mathcal{G}}\left[\ab\left(\vect{\firstlayerWN}^*,\vect{\secondlayer}^*\right)\right] \ge c
\end{align*} for some absolute constant $c > 0$.

While it is hard to characterize $\ab\left(\vect{\firstlayerWN}^*,\vect{\secondlayer}^*\right)$, we find that the set $\tilde\ab(\vv^*, \va^*) \equiv \{\left(\vect{\firstlayerWN}^0,\vect{\secondlayer}^0\right):\left(\vect{\firstlayerWN}^0\right)^\top \vect{\firstlayerWN}^*\ge 0,\left(\vect{\secondlayer}^0\right)^\top \vect{\secondlayer}^* \ge 0, 
\abs{\vect{1}^\top \vect{\secondlayer}^0} \le \abs{\vect{1}^\top \vect{\secondlayer}^*} \}$ is a subset of $\ab(\vv^*, \va^*)$ (c.f. Lemma~\ref{lem:w_angle_converge}-Lemma~\ref{lem:sum_beta_converge}). Furthermore, when the learning rate $\eta$ is sufficiently small, we can design a specific distribution $\mathcal{G}$ so that:
\begin{align*}
\prob_{\left(\vect{\firstlayerWN}^0, \vect{\secondlayer}^0\right)\sim \mathcal{G}}\left[\ab(\vv^*, \va^*)\right] \ge
\prob_{\left(\vect{\firstlayerWN}^0, \vect{\secondlayer}^0\right)\sim \mathcal{G}}\left[\tilde\ab(\vv^*, \va^*)\right] \ge c
\end{align*}

This analysis emphasizes that for non-convex optimization problems, we need to carefully characterize both the trajectory of the algorithm and the initialization.
We believe that this idea is applicable to other non-convex problems.

To obtain the convergence rate, we propose a potential function (also called Lyapunov function in the literature).
For this problem we consider the quantity $\sin^2 \phi^t$ where $\phi^t = \theta\left(\vect{\firstlayerWN}^t,\vect{\firstlayerWN}^*\right)$ and we show it shrinks at a geometric rate (c.f. Lemma~\ref{lem:first_layer_convergence_one_iter}).

\textbf{Organization}
\label{sec:org}
This paper is organized as follows.
In Section~\ref{sec:pre} we introduce the necessary notations and analytical formulas of gradient updates in Algorithm~\ref{algo:weight_normalization_gd}.
In Section~\ref{sec:converge}, we provide our main theorems on the performance of the algorithm and their implications.
In Section~\ref{sec:exp}, we use simulations to verify our theories.
In Section~\ref{sec:proof_sketch}, we give a proof sketch of our main theorem.
We conclude and list future directions in Section~\ref{sec:con}.
We place most of our detailed proofs in the appendix.

\section{Related Works}
\label{sec:rel}

From the point of view of learning theory, it is well known that training a neural network is hard in the worst cases~\citep{blum1989training,livni2014computational,vsima2002training,shalev2017failures,shalev2017weight} and recently,~\citet{shamir2016distribution} showed that assumptions on \emph{both} the target function and the input distribution are needed for optimization algorithms used in practice to succeed.

\textbf{Solve NN without gradient descent.} With some additional assumptions, many works tried to design algorithms that provably learn a neural network with polynomial time and sample complexity~\citep{goel2016reliably,zhang2015learning,sedghi2014provable,janzamin2015beating,goel2017eigenvalue,goel2017learning}.
However these algorithms are specially designed for certain architectures and cannot explain why (stochastic) gradient based optimization algorithms work well in practice.

\textbf{Gradient-based optimization with Gaussian Input.} Focusing on gradient-based algorithms, a line of research analyzed the behavior of (stochastic) gradient descent for \emph{Gaussian} input distribution.
\citet{tian2017analytical} showed that population gradient descent is able to find the true weight vector with random initialization for one-layer one-neuron model.
\citet{soltanolkotabi2017learning} later improved this result by showing the true weights can be exactly recovered by empirical projected gradient descent with enough samples in linear time.
\citet{brutzkus2017globally} showed population gradient descent recovers the true weights of a convolution filter with non-overlapping input in polynomial time.
\citet{zhong2017recovery,zhong2017learning} proved that with sufficiently good initialization, which can be implemented by tensor method, gradient descent can find the true weights of a one-hidden-layer fully connected and convolutional neural network. 
\citet{li2017convergence} showed SGD can recover the true weights of a one-layer ResNet model with ReLU activation under the assumption that the spectral norm of the true weights is within a small constant of the identity mapping.
\cite{panigrahy2018convergence} also analyzed gradient descent for learning a two-layer neural network but with different activation functions.
This paper also follows this line of approach that studies the behavior of gradient descent algorithm with Gaussian inputs.

\textbf{Local minimum and Global minimum.} Finding the optimal weights of a neural network is non-convex problem.
Recently, researchers found that if the objective functions satisfy the following two key properties, (1) all saddle points and local maxima are strict (i.e., there exists a direction with negative curvature), and (2) all local minima are global (no spurious local minmum), then perturbed (stochastic) gradient descent~\citep{ge2015escaping} or methods with second order information~\citep{carmon2016accelerated,agarwal2016finding} can find a global minimum in polynomial time.
\footnote{\citet{lee2016gradient} showed vanilla gradient descent only converges to minimizers with no convergence rates guarantees. Recently, \citet{du2017gradient} gave an exponential time lower bound for the vanilla gradient descent.
In this paper, we give polynomial convergence guarantee on vanilla gradient descent.}
Combined with geometric analyses, these algorithmic results have shown a large number problems, including tensor decomposition~\citep{ge2015escaping}, dictionary learning~\citep{sun2017complete}, matrix sensing~\citep{bhojanapalli2016global,park2017non}, matrix completion~\citep{ge2017no,ge2016matrix} and matrix factorization~\citep{li2016symmetry} can be solved in polynomial time with local search algorithms.

This motivates the research of studying the landscape of neural networks~\citep{kawaguchi2016deep,choromanska2015loss,hardt2016identity,haeffele2015global,mei2016landscape,freeman2016topology,safran2016quality,zhou2017landscape,nguyen2017loss,nguyen2017loss2,ge2017learning,zhou2017landscape,safran2017spurious}.
In particular, \citet{kawaguchi2016deep,hardt2016identity,zhou2017landscape,nguyen2017loss,nguyen2017loss2,feizi2017porcupine} showed that under some conditions, all local minima are global.
Recently, \citet{ge2017learning} showed using a modified objective function satisfying the two properties above, one-hidden-layer neural network can be learned by noisy perturbed gradient descent.
However, for nonlinear activation function, where the number of samples larger than the number of nodes at every layer, which is usually the case in most deep neural network, and natural objective functions like $\ell_2$, it is still unclear whether the strict saddle and ``all locals are global" properties are satisfied.
In this paper, we show that even for a one-hidden-layer neural network with ReLU activation, there exists a spurious local minimum.
However, we further show that randomly initialized local search can achieve \emph{global} minimum with constant probability.

\section{Preliminaries}
\label{sec:pre}
We use bold-faced letters for vectors and matrices.
We use $\norm{\cdot}_2$ to denote the Euclidean norm of a finite-dimensional vector.
We let $\vect{w}^t$ and $\vect{\secondlayer}^t$ be the parameters at the $t$-th iteration and $\vect{w}^*$ and $\vect{\secondlayer}^*$ be the optimal weights.
For two vector $\vect{w}_1$ and $\vect{w}_2$, we use $\theta\left(\vect{w}_1,\vect{w}_2\right)$ to denote the angle between them.
$\secondlayer_i$ is the $i$-th coordinate of $\secondlayer$ and $\mat{Z}_i$ is the transpose of the $i$-th row of $\mat{Z}$ (thus a column vector).
We denote $\mathcal{S}^{p-1}$  the $(p-1)$-dimensional unit sphere and $\mathcal{B}\left(\vect{0},r\right)$ the ball centered at $\vect{0}$ with radius $r$.

In this paper we assume every patch $\mat{Z}_i$ is vector of i.i.d Gaussian random variables.
The following theorem gives an explicit formula for the population loss.
The proof uses basic rotational invariant property and polar decomposition of Gaussian random variables.
See Section~\ref{sec:proof_formula} for details.

\begin{thm}\label{thm:gaussian_input_obj_WN}
If every entry of $\mat{Z}$is  i.i.d. sampled from a Gaussian distribution with mean $0$ and variance $1$, then  population loss is
\begin{align}
\ell\left(\vect{\firstlayerWN},\vect{\secondlayer}\right)= & \frac{1}{2}\left[\frac{\left(\pi-1\right)\norm{\vect{w}^*}_2^2}{2\pi}\norm{\vect{\secondlayer}^*}_2^2 + \frac{\left(\pi-1\right)}{2\pi}\norm{\vect{\secondlayer}}_2^2 \right. \nonumber\\
- & \left.\frac{2\left(g\left(\phi\right)-1\right)\norm{\vect{\firstlayer}^*}_2}{2\pi} \vect{\secondlayer}^\top \vect{\secondlayer}^* +  \frac{\norm{\vect{\firstlayer}^*}_2^2}{2\pi} \left(\vect{1}^\top \vect{\secondlayer}^*\right)^2 \right. \nonumber\\
+ &\left. \frac{1}{2\pi}\left(\vect{1}^\top\vect{\secondlayer}\right)^2 - 2\norm{\vect{\firstlayer}^*}_2\vect{1}^\top\vect{\secondlayer}\cdot\vect{1}^\top\vect{\secondlayer}^*
\right]
\label{eqn:gaussian_input_obj_WN}
\end{align}
where $\phi = \theta\left(\vect{\firstlayerWN},\vect{\firstlayer}^*\right)$ and $g(\phi) = (\pi-\phi)\cos \phi+\sin\phi.$
\end{thm}

Using similar techniques, we can show the gradient also has an analytical form.
\begin{thm}\label{thm:expected_gradient_WN}
Suppose every entry of $\mat{Z}$is  i.i.d. sampled from a Gaussian distribution with mean $0$ and variance $1$. 
Denote $\phi=\theta\left(\vect{w},\vect{w}^*\right)$.
Then the expected gradient of $\vect{w}$ and $\vect{\secondlayer}$ can be written as \begin{align*}
&\expect_{\mat{Z}}\left[\frac{\partial \ell\left(\mat{Z},\vect{\firstlayerWN},\vect{\secondlayer}\right)}{\partial \vect{\firstlayerWN}}\right]\\  = 
&-\frac{1}{2\pi\norm{\vect{\firstlayerWN}}_2}\left(\mat{I}-\frac{\vect{\firstlayerWN}\vect{\firstlayerWN}^\top}{\norm{\vect{\firstlayerWN}}_2^2}\right)\vect{\secondlayer}^\top \vect{\secondlayer}^*\left(\pi-\phi\right)\vect{w}^*\\
&\expect_{\mat{Z}}\left[\frac{\partial \ell\left(\mat{Z},\vect{\firstlayerWN},\vect{\secondlayer}\right)}{\partial \vect{\secondlayer}}\right] \\
= &\frac{1}{2\pi}\left(\vect{1}\vect{1}^\top + \left(\pi-1\right)\mat{I}\right)\vect{\secondlayer} \\
& -\frac{1}{2\pi}\left(\vect{1}\vect{1}^\top + \left(g(\phi)-1\right)\mat{I}\right)\norm{\vect{w}^*}_2\vect{\secondlayer}^*
\end{align*}
\end{thm}
As a remark, if the second layer is fixed, upon proper scaling, the formulas for the population loss and gradient of $\vect{\firstlayerWN}$ are equivalent to the corresponding formulas derived in~\citep{brutzkus2017globally,cho2009kernel}.
However, when the second layer is not fixed, the gradient of $\vect{\firstlayerWN}$ depends on $\vect{\secondlayer}^\top \vect{\secondlayer}^*$, which plays an important role in deciding whether converging to the global or the local minimum.

\section{Main Result}
\label{sec:converge}
We begin with our main theorem about the convergence of gradient descent.
\begin{thm}\label{thm:w_norm_1_gd_converge}
Suppose the initialization satisfies $\left(\vect{\secondlayer}^0\right)^\top \vect{\secondlayer}^* > 0$, $\abs{ \vect{1}^\top\vect{\secondlayer}^0}\le \abs{\vect{1}^\top \vect{\secondlayer}^*}$, $\phi^0 <\pi/2$ and 
step size satisfies \begin{align*}
\eta = O\left( \min\left\{\frac{\left(\vect{\secondlayer}^0\right)^\top \vect{\secondlayer}^*\cos \phi^0}{\left(\norm{\vect{\secondlayer}^*}_2^2+\left(\vect{1}^\top\vect{\secondlayer}^*\right)^2\right)\norm{\vect{\firstlayer}^*}_2^2},\right.\right.\\
\left.\left.\frac{\left(g(\phi_0)-1\right)\norm{\vect{\secondlayer}^*}_2^2\cos \phi^0}{\left(\norm{\vect{\secondlayer}^*}_2^2+\left(\vect{1}^\top\vect{\secondlayer}^*\right)^2\right)\norm{\vect{\firstlayer}^*}_2^2},\right.\right.\\
\left.\left.\frac{\cos \phi^0}{\left(\norm{\vect{\secondlayer}^*}_2^2+\left(\vect{1}^\top\vect{\secondlayer}^*\right)^2\right)\norm{\vect{\firstlayer}^*}_2^2},
\frac{1}{k}\right\}\right).
\end{align*}
Then the convergence of gradient descent has two phases.\\
(\textbf{Phase I: Slow Initial Rate}) 
There exists $T_1 = O\left(\frac{1}{\eta\cos\phi^0\beta^0}+ \frac{1}{\eta}\right)$ such that  we have
$\phi^{T_1} = \Theta\left(1\right)$ and $\left(\vect{\secondlayer}^{T_1}\right)^\top \vect{\secondlayer}^*\norm{\vect{\firstlayer}^*}_2 = \Theta\left(\norm{\vect{\secondlayer}^*}_2^2\norm{\vect{\firstlayer}^*}_2^2\right)$ where $\beta^0 = \min\left\{\left(\vect{\secondlayer}^0\right)^\top \vect{\secondlayer}^*\norm{\vect{\firstlayer}^*}_2, (g(\phi^0)-1)\norm{\vect{\secondlayer}^*}_2^2\norm{\vect{\secondlayer}^*}_2^2\right\}$. \\
(\textbf{Phase II: Fast Rate}) 
Suppose at the $T_1$-th iteration, $\phi^{T_1} = \Theta\left(1\right)$ and $\left(\vect{\secondlayer}^{T_1}\right)^\top \vect{\secondlayer}^*\norm{\vect{\firstlayer}^*}_2 = \Theta\left(\norm{\vect{\secondlayer}^*}_2^2\norm{\vect{\firstlayer}^*}_2^2\right)$, then there exists $T_2 = \widetilde{O}(\left(\frac{1}{\eta\norm{\vect{\firstlayer}^*}_2^2\norm{\vect{\secondlayer}^*}_2^2 } + \frac{1}{\eta}\right)\log\left(\frac{1}{\epsilon}\right))$\footnote{$\widetilde{O}\left(\cdot\right)$ hides logarithmic factors on $\abs{\vect{1}^\top\vect{\secondlayer}^*}\norm{\vect{\firstlayer}^*}_2$ and $\norm{\vect{\secondlayer}^*}_2\norm{\vect{\firstlayer}^*}_2$} such that $\ell\left(\vect{\firstlayerWN}^{T_1+T_2},\vect{\secondlayer}^{T_1+T_2}\right) \le \epsilon\norm{\vect{\firstlayer}^*}_2^2\norm{\vect{\secondlayer}^*}_2^2$.
\end{thm}
Theorem~\ref{thm:w_norm_1_gd_converge} shows under certain conditions of the initialization, gradient descent converges to the global minimum.
The convergence has two phases, at the beginning because the initial signal ($\cos\phi^0 \beta^0$) is small, the convergence is quite slow.
After $T_1$ iterations, the signal becomes stronger and we enter a regime with a faster convergence rate.
See Lemma~\ref{lem:first_layer_convergence_one_iter} for technical details.

Initialization plays an important role in the convergence.
First, Theorem~\ref{thm:w_norm_1_gd_converge} needs the initialization satisfy $\left(\vect{\secondlayer}^0\right)^\top \vect{\secondlayer}^* > 0$, $\abs{ \vect{1}^\top\vect{\secondlayer}^0}\le \abs{\vect{1}^\top \vect{\secondlayer}^*}$ and $\phi^0 <\pi/2$.
Second, the step size $\eta$ and the convergence rate in the first phase also depends on the initialization.
If the initial signal is very small, for example, $\phi^0 \approx \pi/2$ which makes $\cos \phi^0$ close to $0$, we can only choose a very small step size and because $T_1$ depends on the inverse of $\cos \phi^0$, we need a large number of iterations to enter phase II.
We provide the following initialization scheme which ensures the conditions required by Theorem~\ref{thm:w_norm_1_gd_converge} and a large enough initial signal.

\begin{thm}\label{thm:init}
Let $\vect{\firstlayerWN} \sim \unif\left(\mathcal{S}^{p-1}\right)$ and $\vect{\secondlayer} \sim \unif\left( \mathcal{B}\left(\vect{0},\frac{\abs{\vect{1}^\top\vect{\secondlayer}^*}\norm{\vect{\firstlayer}^*}_2}{\sqrt{k}}\right)\right)$, then exists \[\left(\vect{\firstlayerWN}^0,\vect{\secondlayer}^0\right) \in \left\{\left(\vect{\firstlayerWN},\vect{\secondlayer}\right), \left(\vect{\firstlayerWN},-\vect{\secondlayer}\right), \left(-\vect{\firstlayerWN},\vect{\secondlayer}\right), \left(-\vect{\firstlayerWN},-\vect{\secondlayer}\right)  \right\}\] that $\left(\vect{\secondlayer}^0\right)^\top \vect{\secondlayer}^* > 0$, $\abs{ \vect{1}^\top\vect{\secondlayer}^0}\le \abs{\vect{1}^\top \vect{\secondlayer}^*}$ and $\phi^0 <\pi/2$.
Further, with high probability, the initialization satisfies $\left(\vect{\secondlayer}^0\right)^\top \vect{\secondlayer}^*\norm{\vect{\firstlayer}^*}_2 = \Theta\left(\frac{\abs{\vect{1}^\top\vect{\secondlayer}^*}\norm{\vect{\secondlayer}^*}_2\norm{\vect{\firstlayer}^*}_2^2}{k}\right)$, and $\phi^0 = \Theta\left(\frac{1}{\sqrt{p}}\right)$.
\end{thm}
Theorem~\ref{thm:init} shows after generating a pair of random vectors $(\vect{\firstlayerWN},\vect{\secondlayer})$, trying out all $4$ sign combinations of $(\vect{\firstlayerWN},\vect{\secondlayer})$, we can find the global minimum by gradient descent.
Further, because the initial signal is not too small, we only need to set the step size to be $O(1/\poly(k,p,\norm{\vect{\firstlayer}^*}_2\norm{\vect{\secondlayer}}_2))$ and the number of iterations in phase I is at most $O(\poly(k,p,\norm{\vect{\firstlayer}^*}_2\norm{\vect{\secondlayer}}_2))$.
Therefore, Theorem~\ref{thm:w_norm_1_gd_converge} and Theorem~\ref{thm:init} together show that randomly initialized gradient descent learns an one-hidden-layer convolutional neural network in polynomial time.
The proof of the first part of Theorem~\ref{thm:init} uses the symmetry of unit sphere and ball and the second part is a standard application of random vector in high-dimensional spaces.
See Lemma 2.5 of~\citep{hardt2014noisy} for example.

\noindent \textbf{Remark 1:} For the second layer we use $O\left(\frac{1}{\sqrt{k}}\right)$ type initialization, verifying common initialization techniques~\citep{glorot2010understanding,he2015delving,lecun1998efficient}.

\noindent \textbf{Remark 2:} The Gaussian input assumption is not necessarily true in practice, although this is a common assumption appeared in the previous papers~\citep{brutzkus2017globally,li2017convergence,zhong2017learning,zhong2017recovery,tian2017analytical,xie2017diverse,shalev2017weight} and also considered plausible in~\citep{choromanska2015loss}. 
Our result can be easily generalized to rotation invariant distributions.
However, extending to more general distributional assumption, e.g., structural conditions used in~\citep{du2017convolutional} remains a challenging open problem.

\noindent \textbf{Remark 3:} Since we  only require initialization to be smaller than some quantities of $\vect{a}^*$ and $\vect{w}^*$. In practice, if the optimization fails, i.e., the initialization is too large, one can halve the initialization size, and eventually these conditions will be met.

\subsection{Gradient Descent Can Converge to the Spurious Local Minimum}
Theorem~\ref{thm:init} shows that among $\left\{\left(\vect{\firstlayerWN},\vect{\secondlayer}\right), \left(\vect{\firstlayerWN},-\vect{\secondlayer}\right), \left(-\vect{\firstlayerWN},\vect{\secondlayer}\right), \left(-\vect{\firstlayerWN},-\vect{\secondlayer}\right)  \right\}$, there is a  pair that enables gradient descent to converge to the global minimum. 
Perhaps surprisingly, the next theorem shows that under some conditions of the underlying truth, there is also a pair that makes gradient descent converge to the spurious local minimum.

\begin{thm}\label{thm:w_norm_1_gd_converge_bad}
Without loss of generality, we let $\norm{\vect{\firstlayer}^*}_2=1$.
Suppose $\left(\vect{1}^\top \vect{\secondlayer}^*\right)^2 < \frac{1}{\poly(p)}\norm{\vect{\secondlayer}^*}_2^2$ and $\eta$ is sufficiently small.
Let $\vect{\firstlayerWN} \sim \unif\left(\mathcal{S}^{p-1}\right)$ and $\vect{\secondlayer} \sim \unif\left( \mathcal{B}\left(\vect{0},\frac{\abs{\vect{1}^\top\vect{\secondlayer}^*}}{\sqrt{k}}\right)\right)$, then with high probability, there exists $\left(\vect{\firstlayerWN}^0,\vect{\secondlayer}^0\right) \in \left\{\left(\vect{\firstlayerWN},\vect{\secondlayer}\right), \left(\vect{\firstlayerWN},-\vect{\secondlayer}\right), \left(-\vect{\firstlayerWN},\vect{\secondlayer}\right), \left(-\vect{\firstlayerWN},-\vect{\secondlayer}\right)  \right\}$ that $\left(\vect{\secondlayer}^0\right)^\top \vect{\secondlayer}^* < 0$, $\abs{ \vect{1}^\top\vect{\secondlayer}^0}\le \abs{\vect{1}^\top \vect{\secondlayer}^*}$, $g\left(\phi^0\right) \le  \frac{-2\left(\vect{1}^\top\vect{\secondlayer}^*\right)^2}{\norm{\vect{\secondlayer}^*}_2^2} + 1$. 
If $\left(\vect{\firstlayerWN}^0,\vect{\secondlayer}^0\right)$ is used as the initialization, when  Algorithm~\ref{algo:weight_normalization_gd} converges, we have
\begin{align*}
\theta\left(\vect{\firstlayerWN},\vect{\firstlayer}^*\right) = \pi, 
\vect{\secondlayer} = \left(\vect{1}\vect{1}^\top + \left(\pi-1\right)\mat{I}\right)^{-1}\left(\vect{1}\vect{1}^\top - \mat{I}\right)\vect{\secondlayer}^* 
\end{align*} and
$\ell\left(\vect{\firstlayerWN},\vect{\secondlayer}\right) = \Omega\left(\norm{\vect{\secondlayer}^*}_2^2\right)$.
\end{thm}

Unlike Theorem~\ref{thm:w_norm_1_gd_converge} which requires no assumption on the underlying truth $\vect{\secondlayer}^*$, Theorem~\ref{thm:w_norm_1_gd_converge_bad} assumes  $\left(\vect{1}^\top \vect{\secondlayer}^*\right)^2 < \frac{1}{\poly(p)}\norm{\vect{\secondlayer}^*}_2^2$.
This technical condition comes from the proof which requires invariance $g(\phi^t) \le \frac{-2\left(\vect{1}^\top \vect{\secondlayer}^*\right)^2}{\norm{\vect{\secondlayer}^*}_2^2}$ for all iterations.
To ensure there exists $\left(\vect{\firstlayerWN}^0,\vect{\secondlayer}^0\right)$ 
which makes $g(\phi^0) \le \frac{-2\left(\vect{1}^\top \vect{\secondlayer}^*\right)^2}{\norm{\vect{\secondlayer}^*}_2^2}$, we need $\frac{\left(\vect{1}^\top \vect{\secondlayer}^*\right)^2}{\norm{\vect{\secondlayer}^*}_2^2}$ relatively small.
See Section~\ref{sec:proof_local_min} for more technical insights.

A natural question is whether the ratio $\frac{\left(\vect{1}^\top \vect{\secondlayer}^*\right)^2}{\norm{\vect{\secondlayer}^*}_2^2}$ becomes larger, the probability randomly gradient descent converging to the global minimum, becomes larger as well.
We verify this phenomenon empirically in Section~\ref{sec:exp}.

\section{Proof Sketch}
\label{sec:proof_sketch}
In Section~\ref{sec:proof_sketch_qualitative}, we give qualitative high level intuition on why the initial conditions are sufficient for gradient descent to converge to the global minimum.
In Section~\ref{sec:proof_sketch_quantitative}, we explain why the gradient descent has two phases.

\subsection{Qualitative Analysis of Convergence}
\label{sec:proof_sketch_qualitative}
The convergence to global optimum relies on a geometric characterization of saddle points and a series of invariants throughout the gradient descent dynamics.
The next lemma gives the analysis of stationary points.
The main step is to check the first order condition of stationary points using Theorem~\ref{thm:expected_gradient_WN}.
\begin{lem}[Stationary Point Analysis]\label{lem:stationary_point}
When the gradient descent converges, $\vect{\secondlayer}^\top \vect{\secondlayer}^* \neq 0$ and $\norm{\vect{\firstlayerWN}}_2 < \infty$, we have either \begin{align*}
\theta\left(\vect{\firstlayerWN},\vect{\firstlayer}^*\right) &= 0, \vect{\secondlayer} = \norm{\vect{\firstlayer}^*}_2\vect{\secondlayer}^* \\
\text{ or }
\theta\left(\vect{\firstlayerWN},\vect{\firstlayer}^*\right) &= \pi, \\\vect{\secondlayer} &= \left(\vect{1}\vect{1}^\top + \left(\pi-1\right)\mat{I}\right)^{-1}\left(\vect{1}\vect{1}^\top - \mat{I}\right)\norm{\vect{\firstlayer}^*}_2\vect{\secondlayer}^*.
\end{align*}
\end{lem}
This lemma shows that when the algorithm converges, and $\vect{\secondlayer}$ and $\vect{\secondlayer}^*$ are not orthogonal, then we arrive at either a global optimal point or a local minimum.
Now recall the gradient formula of $\vect{\firstlayerWN}$:
$\frac{\partial \ell\left(\vect{\firstlayerWN},\vect{\secondlayer}\right)}{\partial \vect{\firstlayerWN}}= 
-\frac{1}{2\pi\norm{\vect{\firstlayerWN}}_2}\left(\mat{I}-\frac{\vect{\firstlayerWN}\vect{\firstlayerWN}^\top}{\norm{\vect{\firstlayerWN}}_2^2}\right)\vect{\secondlayer}^\top \vect{\secondlayer}^*\left(\pi-\phi\right)\vect{w}^*$.
Notice that $\phi \le \pi$ and $\left(\mat{I}-\frac{\vect{\firstlayerWN}\vect{\firstlayerWN}^\top}{\norm{\vect{\firstlayerWN}}_2^2}\right)$ is just the projection matrix onto the complement of $\vect{\firstlayerWN}$.
Therefore, the sign of inner product between $\vect{\secondlayer}$ and $\vect{\secondlayer}^*$ plays a crucial role in the dynamics of Algorithm~\ref{algo:weight_normalization_gd} because if the inner product is positive, the gradient update will decrease the angle between $\vect{\firstlayerWN}$ and $\vect{\firstlayer}^*$ and if it is negative, the angle will increase.
This observation is formalized in the lemma below.
\begin{lem}[Invariance I: Tje Angle between $\vect{\firstlayerWN}$ and $\vect{\firstlayer}^*$ always decreases.]\label{lem:w_angle_converge}
	If $\left(\vect{\secondlayer}^t\right)^\top \vect{\secondlayer}^* > 0$, then $\phi^{t+1} \le \phi^{t}$.
\end{lem}
This lemma shows that when $\left(\vect{\secondlayer}^t\right)^\top \vect{\secondlayer}^* > 0$ for all $t$, gradient descent converges to the global minimum.
Thus, we need to study the dynamics of $\left(\vect{\secondlayer}^t\right)^\top \vect{\secondlayer}^*$.
For the ease of presentation, without loss of generality, we assume $\norm{\vect{\firstlayer}^*}_2= 1$.
By the gradient formula of $\vect{\secondlayer}$, we have \begin{align}
	&\left(\vect{\secondlayer}^{t+1}\right)^\top\vect{\secondlayer}^* \nonumber \\ = &\left(1-\frac{\eta(\pi-1)}{2\pi}\right)\left(\vect{\secondlayer}^t\right)^\top \vect{\secondlayer}^* + \frac{\eta(g(\phi^t)-1)}{2\pi}\norm{\vect{\secondlayer}^t}_2^2 \nonumber\\
	&+ \frac{\eta}{2\pi}\left(\left(\vect{1}^\top \vect{\secondlayer}^*\right)^2 - \left(\vect{1}^\top \vect{\secondlayer}^t\right)\left(\vect{1}^\top \vect{\secondlayer}^*\right)\right). \label{eqn:inner_product_dynamics}
\end{align}
We can use induction to prove the invariance.
If $\left(\vect{\secondlayer}^t\right)^\top \vect{\secondlayer}^* > 0$ and $\phi^t < \frac{\pi}{2}$ the first term of Equation~\eqref{eqn:inner_product_dynamics} is non-negative.
For the second term, notice that if $\phi^t <\frac{\pi}{2}$, we have $g(\phi^t) > 1$, so the second term is non-negative.
Therefore, as long as $\left(\left(\vect{1}^\top \vect{\secondlayer}^*\right)^2 - \left(\vect{1}^\top \vect{\secondlayer}^t\right)\left(\vect{1}^\top \vect{\secondlayer}^*\right)\right)$ is also non-negative, we have the desired invariance.
The next lemma summarizes the above analysis.
\begin{lem} [Invariance II: Positive Signal from the Second Layer.]\label{lem:beta_inner_pos}
	If $\left(\vect{\secondlayer}^t\right)^\top \vect{\secondlayer}^* > 0$, $0\le \vect{1}^\top \vect{\secondlayer}^*\cdot\vect{1}^\top\vect{\secondlayer}^t \le \left(\vect{1}^\top \vect{\secondlayer}^*\right)^2$, $0 < \phi^t <\pi/2 $ and $\eta < 2$, then $\left(\vect{\secondlayer}^{t+1}\right)^\top\vect{\secondlayer}^* > 0$.
\end{lem}
It remains to prove $\left(\left(\vect{1}^\top \vect{\secondlayer}^*\right)^2 - \left(\vect{1}^\top \vect{\secondlayer}^t\right)\left(\vect{1}^\top \vect{\secondlayer}^*\right)\right) > 0$.
Again, we study the dynamics of this quantity.
Using the gradient formula and some algebra, we have \begin{align*}
\vect{1}^\top \vect{\secondlayer}^{t+1} \cdot \vect{1}^\top \vect{\secondlayer}^*
&\le \left(1-\frac{\eta\left(k-\pi-1\right)}{2\pi}\right)\vect{1}^\top \vect{\secondlayer}^{t}\cdot\vect{1}^\top\vect{\secondlayer}^* \\
&+ \frac{\eta\left(k+g(\phi^t)-1\right)}{2}\left(\vect{1}^\top \vect{\secondlayer}^*\right)^2 \\
& \le \left(1-\frac{\eta\left(k-\pi-1\right)}{2\pi}\right)\vect{1}^\top \vect{\secondlayer}^{t}\cdot\vect{1}^\top\vect{\secondlayer}^* \\
&+ \frac{\eta\left(k+\pi-1\right)}{2}\left(\vect{1}^\top \vect{\secondlayer}^*\right)^2
\end{align*}
where have used the fact that $g(\phi)\le \pi$ for all $0 \le \phi \le  \frac{\pi}{2}$.
Therefore we have \begin{align*}
	&\left(\vect{1}^\top\vect{\secondlayer}^*-\vect{1}^\top\vect{\secondlayer}^{t+1}\right)\cdot\vect{1}^\top \vect{\secondlayer}^*\\
	 \ge &\left(1-\frac{\eta(k+\pi-1)}{2\pi}\right)\left(\vect{1}^\top\vect{\secondlayer}^*-\vect{1}^\top\vect{\secondlayer}^{t}\right) \vect{1}^\top \vect{\secondlayer}^*.
\end{align*}
These imply the third invariance.
\begin{lem}[Invariance III: Summation of Second Layer Always Small.]\label{lem:sum_beta_converge}
If $\vect{1}^\top \vect{\secondlayer}^*\cdot\vect{1}^\top\vect{\secondlayer}^t \le \left(\vect{1}^\top \vect{\secondlayer}^*\right)^2$ and $\eta < \frac{2\pi}{k+\pi-1}$ then $\vect{1}^\top \vect{\secondlayer}^*\cdot\vect{1}^\top\vect{\secondlayer}^{t+1} \le \left(\vect{1}^\top \vect{\secondlayer}^*\right)^2$.
\end{lem}
To sum up, if the initialization satisfies (1) $\phi^0 < \frac{\pi}{2}$, (2) $\left(\vect{\secondlayer}^0\right)^\top \vect{\secondlayer}^* > 0$ and (3) $\vect{1}^\top \vect{\secondlayer}^*\cdot\vect{1}^\top\vect{\secondlayer}^0 \le \left(\vect{1}^\top \vect{\secondlayer}^*\right)^2$, with Lemma~\ref{lem:w_angle_converge},~\ref{lem:beta_inner_pos},~\ref{lem:sum_beta_converge}, by induction we can show the convergence to the global minimum.
Further, Theorem~\ref{thm:init} shows these three conditions are true with constant probability using random initialization.

\subsection{Quantitative Analysis of Two Phase Phenomenon}
\label{sec:proof_sketch_quantitative}
In this section we demonstrate why there is a two-phase phenomenon.
Throughout this section, we assume the conditions in Section~\ref{sec:proof_sketch_qualitative} hold. 
We first consider the convergence of the first layer.
Because we are using weight-normalization, only the angle between $\vect{\firstlayerWN}$ and $\vect{\firstlayer}^*$ will affect the prediction.
Therefore, in this paper, we study the dynamics $\sin^2\phi^t$.
The following lemma quantitatively characterize the shrinkage of this quantity of one iteration.
\begin{lem}[Convergence of Angle between $\vect{\firstlayerWN}$ and $\vect{\firstlayer}^*$]\label{lem:first_layer_convergence_one_iter}
	Under the same assumptions as in Theorem~\ref{thm:w_norm_1_gd_converge}.
	Let $\beta^0 = \min\left\{\left(\vect{\secondlayer}^0\right)^\top \vect{\secondlayer}^*,\left(g(\phi^0)-1\right)\norm{\vect{\secondlayer}^*}_2^2\right\}\norm{\vect{\firstlayer}^*}_2^2$.
	If the step size satisfies 
	$\eta = O
	( \min\{\frac{\beta^0\cos\phi^0}{\left(\norm{\vect{\secondlayer}^*}_2^2+\left(\vect{1}^\top \vect{\secondlayer}^*\right)^2\right)\norm{\vect{\firstlayer}^*}_2^2}, \frac{\cos\phi^0}{\left(\norm{\vect{\secondlayer}^*}_2^2+\left(\vect{1}^\top \vect{\secondlayer}^*\right)^2\right)\norm{\vect{\firstlayer}^*}_2^2},\frac{1}{k}\})$, we have \begin{align*}
	\sin^2\phi^{t+1} \le \left(1-\eta\cos\phi^t\lambda^t\right)\sin^2\phi^t
	\end{align*}  where $\lambda^t = \frac{\norm{\vect{\firstlayer}^*}_2\left(\pi-\phi^t\right)\left(\vect{\secondlayer}^t\right)^\top\vect{\secondlayer}^*}{2\pi\norm{\vect{\firstlayerWN}^t}_2^2}$.
\end{lem}
This lemma shows the convergence rate depends on two crucial quantities, $\cos \phi^t$ and $\lambda^t$.
At the beginning, both $\cos\phi^t$ and $\lambda^t$ are small.
Nevertheless, Lemma~\ref{lem:firstlayerWN_norm_always_small} shows $\lambda^t$ is universally lower bounded by $\Omega\left(\beta^0\right)$.
Therefore, after $O(\frac{1}{\eta\cos\phi^0 \beta^0})$ we have $\cos\phi^t = \Omega\left(1\right)$.
Once $\cos\phi^t = \Omega\left(1\right)$, Lemma~\ref{cor:inner_prduct_convergences} shows, after $O\left(\frac{1}{\eta}\right)$ iterations, $\left(\vect{\secondlayer}^t\right)\vect{\secondlayer}^*\norm{\vect{\firstlayer}^*} = \Omega\left(\norm{\vect{\firstlayer}^*}_2^2\norm{\vect{\secondlayer}^*}_2^2\right)$.
Combining the facts $\norm{\vect{\firstlayerWN}^t}_2 \le 2$ (Lemma~\ref{lem:firstlayerWN_norm_always_small}) and $\phi^t < \pi/2$, we have $\cos\phi^t\lambda^t = \Omega\left(\norm{\vect{\firstlayer}^*}_2^2\norm{\vect{\secondlayer}^*}_2^2\right)$.
Now we enter phase II. 

In phase II, Lemma~\ref{lem:first_layer_convergence_one_iter} shows \[
\sin^2\phi^{t+1} \le \left(1-\eta C \norm{\vect{\firstlayer}^*}_2^2\norm{\vect{\secondlayer}^*}_2^2\right)\sin^2\phi^t
\] for some positive absolute constant $C$.
Therefore, we have much faster convergence rate than that in the Phase I.
After only $\widetilde{O}\left(\frac{1}{\eta \norm{\vect{\firstlayer}^*}_2^2\norm{\vect{\secondlayer}}_2^2}\log\left(\frac{1}{\epsilon}\right)\right)$ iterations, we obtain $\phi \le \epsilon$.

Once we have this, we can use  Lemma~\ref{lem:second_layer_sum_convergence} to show $\abs{\vect{1}^\top \vect{\secondlayer}^*-\vect{1}^\top \vect{\secondlayer}} \le O\left(\epsilon \norm{\vect{\secondlayer}^*}_2\right)$ after $\widetilde{O}(\frac{1}{\eta k}\log\left(\frac{1}{\epsilon}\right))$ iterations. 
Next, using Lemma~\ref{lem:second_layer_convergence}, we can show after $\widetilde{O}\left(\frac{1}{\eta}\log\frac{1}{\epsilon}\right)$ iterations, $\norm{\vect{\secondlayer}-\vect{\secondlayer}^*}_2 = O\left(\epsilon\norm{\vect{\secondlayer}^*}_2\right)$.
Lastly, 
Lemma~\ref{lem:prediction_error} shows if $\norm{\vect{\secondlayer}-\vect{\secondlayer}^*}_2 = O\left(\epsilon\norm{\vect{\secondlayer}^*}_2\right)$ and $\phi= O\left( \epsilon\right)$ we have we have $\ell\left(\vect{\firstlayerWN},\vect{\secondlayer}\right) = O\left(\epsilon \norm{\vect{\secondlayer}^*}_2^2\right)$.

\section{Experiments}
\label{sec:exp}
In this section, we illustrate our theoretical results with numerical experiments.
Again without loss of generality, we assume $\norm{\vect{\firstlayer}^*}_2=1$ in this section.

\subsection{Multi-phase Phenomenon}
In Figure~\ref{fig:dynamics}, we set $k=20$, $p=25$ and we consider 4 key quantities in proving Theorem~\ref{thm:w_norm_1_gd_converge}, namely, angle between $\vect{\firstlayerWN}$ and $\vect{\firstlayer}^*$ (c.f. Lemma~\ref{lem:first_layer_convergence_one_iter}), $\norm{\vect{\secondlayer}-\vect{\secondlayer}^*}$ (c.f. Lemma~\ref{lem:second_layer_convergence}), $\abs{\vect{1}^\top \vect{\secondlayer}-\vect{1}^\top \vect{\secondlayer}^*}$ (c.f. Lemma~\ref{lem:second_layer_sum_convergence}) and prediction error (c.f. Lemma~\ref{lem:prediction_error}).

When we achieve the global minimum, all these quantities are $0$.
At the beginning (first $~\sim 10$ iterations), $\abs{\vect{1}^\top \vect{\secondlayer}-\vect{1}^\top \vect{\secondlayer}^*}$ and the prediction error drop quickly.
This is because for the gradient of $\vect{\secondlayer}$, $\vect{1}\vect{1}^\top\vect{\secondlayer}^*$ is the dominating term which will make $\vect{1}\vect{1}^\top\vect{\secondlayer}$ closer to  $\vect{1}\vect{1}^\top\vect{\secondlayer}^*$ quickly.

After that, for the next $\sim 200$ iterations, all quantities decrease at a slow rate.
This phenomenon is explained to the Phase I stage in Theorem~\ref{thm:w_norm_1_gd_converge}.
The rate is slow because the initial signal is small.

After $\sim 200$ iterations, all quantities drop at a much faster rate.
This is because the signal is very strong and since the convergence rate is proportional to this signal, we have a much faster convergence rate (c.f. Phase II of Theorem~\ref{thm:w_norm_1_gd_converge}).

\begin{figure}[tb]
\centering
\includegraphics[width=\linewidth]{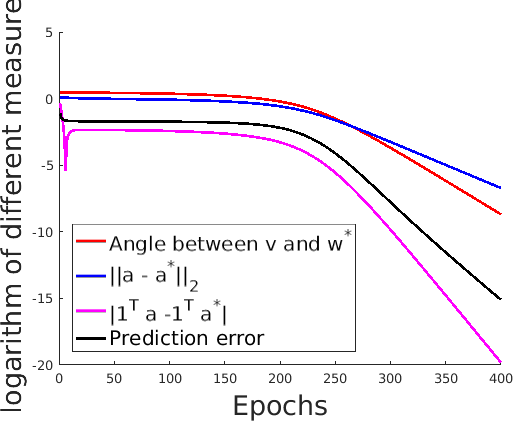}
	\caption{Convergence of different measures we considered in proving Theorem~\ref{thm:w_norm_1_gd_converge}.
	In the first $\sim200$ iterations, all quantities drop slowly.
	After that, these quantities converge at much faster linear rates.
}\label{fig:dynamics}
\end{figure}

\subsection{Probability of Converging to the Global Minimum}
In this section we test the probability of converging to the global minimum using the random initialization scheme described in Theorem~\ref{thm:init}.
We set $p=6$ and vary $k$ and $\frac{(\vect{1}^\top\vect{\secondlayer}^*)^2}{\norm{\vect{\secondlayer}}_2^2}$.
We run 5000 random initializations for each $(k,\frac{(\vect{1}^\top\vect{\secondlayer}^*)^2}{\norm{\vect{\secondlayer}}_2^2})$ and compute the probability of converging to the global minimum.

In Theorem~\ref{thm:w_norm_1_gd_converge_bad}, we showed if $\frac{(\vect{1}^\top\vect{\secondlayer}^*)^2}{\norm{\vect{\secondlayer}}_2^2}$ is sufficiently small, randomly initialized gradient descent converges to the spurious local minimum with constant probability.
Table~\ref{tab:success_prob} empirically verifies the importance of this assumption.
For every fixed $k$ if  $\frac{(\vect{1}^\top\vect{\secondlayer}^*)^2}{\norm{\vect{\secondlayer}}_2^2}$ becomes larger, the probability of converging to the global minimum becomes larger.

An interesting phenomenon is for every fixed ratio  $\frac{(\vect{1}^\top\vect{\secondlayer}^*)^2}{\norm{\vect{\secondlayer}}_2^2}$ when $k$ becomes lager, the probability of converging to the global minimum becomes smaller.
How to quantitatively characterize  the relationship between the success probability and the dimension of the second layer is an open problem.

\section{Conclusion and Future Works}
\label{sec:con}
In this paper we proved the first polynomial convergence guarantee of randomly initialized gradient descent algorithm for learning a one-hidden-layer convolutional neural network.
Our result reveals an interesting phenomenon that randomly initialized local search algorithm can converge to a global minimum or a spurious local minimum.
We give a quantitative characterization of gradient descent dynamics to explain the two-phase convergence phenomenon.
Experimental results also verify our theoretical findings.
Here we list some future directions.

Our analysis focused on the population loss with Gaussian input. 
In practice one uses (stochastic) gradient descent on the empirical loss.
Concentration results in~\citep{mei2016landscape,soltanolkotabi2017learning} are useful to generalize our results to the empirical version.
A more challenging question is how to extend the analysis of gradient dynamics beyond rotationally invariant input distributions.
\citet{du2017convolutional} proved the convergence of gradient descent under some structural input distribution assumptions in the one-layer convolutional neural network.
It would be interesting to bring their insights to our setting.

Another interesting direction is to generalize our result to deeper and wider architectures. 
Specifically, an open problem is under what conditions randomly initialized gradient descent algorithms can learn one-hidden-layer fully connected neural network or a convolutional neural network with multiple kernels. 
Existing results often require sufficiently good initialization~\citep{zhong2017learning,zhong2017recovery}.
We believe the insights from this paper, especially the invariance principles in Section~\ref{sec:proof_sketch_qualitative} are helpful to understand the behaviors of gradient-based algorithms in these settings.

\begin{table}[tb]
	\centering
	\resizebox{\columnwidth}{!}{%
	\renewcommand{\arraystretch}{1.5}
	\begin{tabular}{ |c|c|c|c|c|c|c|}
		\hline
		\backslashbox{k}{$\frac{(\vect{1}^\top\vect{\secondlayer}^*)^2}{\norm{\vect{\secondlayer}^*}_2^2}$} & 0 & 1 & 4 & 9
		 & 16 & 25\\ 
		\hline
		25 & 0.50 & 0.55 & 0.73 & 1  & 1 & 1\\
		\hline
		36 &  0.50 &  0.53 & 0.66 & 0.89  & 1 & 1\\
		\hline
		49 &  0.50 &  0.53 & 0.61 & 0.78  & 1 & 1 \\
		\hline
		64 &  0.50 &  0.51 & 0.59 & 0.71  & 0.89 & 1 \\
		\hline
		 81&  0.50 &  0.53 & 0.57 & 0.66  & 0.81 & 0.97\\
		\hline
		100&  0.50 &  0.50 & 0.57 & 0.63  & 0.75 & 0.90\\
		\hline
	\end{tabular}
	}
	\caption{Probability of converging to the global minimum with different $\frac{(\vect{1}^\top\vect{\secondlayer}^*)^2}{\norm{\vect{\secondlayer}}_2^2}$ and $k$.
	For every fixed $k$, when $\frac{(\vect{1}^\top\vect{\secondlayer}^*)^2}{\norm{\vect{\secondlayer}}_2^2}$ becomes larger, the probability of converging to the global minimum becomes larger and for every fixed ratio  $\frac{(\vect{1}^\top\vect{\secondlayer}^*)^2}{\norm{\vect{\secondlayer}}_2^2}$ when $k$ becomes lager, the probability of converging to the global minimum becomes smaller.
		\label{tab:success_prob}
	}
\end{table}

\section{Acknowledgment}
\label{sec:ack}
This research was partly funded by NSF grant IIS1563887, AFRL grant FA8750-17-2-0212 DARPA D17AP00001.
J.D.L. acknowledges support of the ARO under MURI Award W911NF-11-1-0303.  This is part of the collaboration between US DOD, UK MOD and UK Engineering and Physical Research Council (EPSRC) under the Multidisciplinary University Research Initiative.
The authors thank Xiaolong Wang and Kai Zhong for useful discussions.

\bibliography{simonduref}
\bibliographystyle{icml2018}
\onecolumn
\newpage
\appendix
\section{Proofs of Section~\ref{sec:pre}}
\label{sec:proof_formula}
\begin{proof}[Proof of Theorem~~\ref{thm:gaussian_input_obj_WN}]
We first expand the loss function directly.
\begin{align*}
&\ell\left(\vect{\firstlayerWN},\vect{\secondlayer}\right)\\
= &\expect\left[\frac12 \left(y-\vect{\secondlayer}^\top\relu{\mat{Z}}\vect{\firstlayer}\right)^2\right] \\
=&\left(\vect{\secondlayer}^*\right)^\top\expect\left[\relu{\mat{Z}\vect{\firstlayer}^*}\relu{\mat{Z}\vect{\firstlayer}^*}^\top\right]\vect{\secondlayer}^* 
+\vect{\secondlayer}^\top\expect\left[\relu{\mat{Z}\vect{\firstlayer}}\relu{\mat{Z}\vect{\firstlayer}}^\top\right]\vect{\secondlayer} - 2\vect{\secondlayer}^\top \expect\left[\relu{\mat{Z}\vect{\firstlayer}}\relu{\mat{Z}\vect{\firstlayer}^*}^\top\right]\vect{\secondlayer}^* \\
= & \left(\vect{\secondlayer}^*\right)^\top \mat{A}\left(\vect{\firstlayer}^*\right)\vect{\secondlayer}^* + \vect{\secondlayer}^\top\mat{A}\left(\vect{\firstlayer}\right)\vect{\secondlayer} - 2\vect{\secondlayer}^\top\mat{B}\left(\vect{\firstlayer},\vect{\firstlayer}^*\right) \vect{w}^*.
\end{align*}
where for simplicity, we denote \begin{align}\mat{A}(\vect{w}) =& \expect\left[\relu{\mat{Z}\vect{w}}\relu{\mat{Z}\vect{w}}^\top\right] \label{eqn:A}\\
\mat{B}\left(\vect{w},\vect{w}^*\right) = &\expect\left[
\relu{\mat{Z}\vect{w}}\relu{\mat{Z}\vect{w}^*}^\top
\right] \label{eqn:B}.
\end{align}
For $i\neq j$, using the second identity of Lemma~\ref{lem:gaussian_basic_facts}, we can compute \[\mat{A}(\vect{w})_{ij} = \expect\left[\relu{\mat{Z}_i^\top\vect{w}}\right]\expect\left[\relu{\mat{Z}_j^\top\vect{w}}\right] = \frac{1}{2\pi}\norm{\vect{w}}_2^2\] 
For $i=j$,  using the second moment formula of half-Gaussian distribution we can compute \[\mat{A}\left(\vect{w}\right)_{ii} = \frac{1}{2}\norm{\vect{w}}_2^2.\] 
Therefore \[\mat{A}(\vect{w}) = \frac{1}{2\pi}\norm{\vect{w}}_2^2\left(\vect{1}\vect{1}^\top + \left(\pi-1\right)\mat{I}\right).\]
Now let us compute $\mat{B}\left(\vect{w},\vect{w}_*\right)$.
For $i\neq j$, similar to $\mat{A}(\vect{w})_{ij}$, using the independence property of Gaussian, we have \[\mat{B}\left(\vect{w},\vect{w}_*\right)_{ij} = \frac{1}{2\pi}\norm{\vect{w}}_2\norm{\vect{w}^*}_2.\]
Next, using the fourth identity of Lemma~\ref{lem:gaussian_basic_facts}, we have \[\mat{B}\left(\vect{w},\vect{w}^*\right)_{ii} = \frac{1}{2\pi}\left(\cos\phi\left(\pi-\phi\right)+\sin\phi\right)\norm{\vect{w}}_2\norm{\vect{w}^*}_2.\]
Therefore, we can also write $\mat{B}\left(\vect{w},\vect{w}^*\right)$ in a compact form\[
\mat{B}\left(\vect{w},\vect{w}^*\right) = \frac{1}{2\pi}\norm{\vect{w}}_2\norm{\vect{w}^*}_2\left(\vect{1}\vect{1}^\top + \left(\cos\phi\left(\pi-\phi\right)+\sin\phi-1\right)\mat{I}\right).
\]
Plugging in the formulas of $\mat{A}(\vect{w})$ and $\mat{B}\left(\vect{w},\vect{w}^*\right)$ and $\vect{\firstlayer} = \frac{\vect{\firstlayerWN}}{\norm{\vect{\firstlayerWN}}_2}$, we obtain the desired result.
\end{proof}

\begin{proof}[Proof of Theorem~\ref{thm:expected_gradient_WN}]
We first compute the expect gradient for $\vect{\firstlayerWN}$.
From\citep{salimans2016weight}, we know\begin{align*}
\frac{\partial \ell\left(\vect{\firstlayerWN},\vect{\secondlayer}\right)}{\partial \vect{\firstlayerWN}} = \frac{1}{\norm{\vect{\firstlayerWN}}_2}\left(\mat{I}-\frac{\vect{\firstlayerWN}\vect{\firstlayerWN}^\top}{\norm{\vect{\firstlayerWN}}_2^2}\right)\frac{\partial \ell\left(\vect{\firstlayer},\vect{\secondlayer}\right)}{\partial \vect{\firstlayer}}.
\end{align*}
Recall the gradient formula, \begin{align}
&\frac{\partial \ell\left(\mat{Z},\vect{w},\vect{\secondlayer}\right)}{\partial \vect{w}}\nonumber\\  = &\left(\sum_{i=1}^{k}\secondlayer_i^*\relu{\mat{Z}_i\vect{w}}-\sum_{i=1}^{k}\secondlayer^*_i\relu{\mat{Z}\vect{w}^*}\right)\left(\sum_{i=1}^{k}\secondlayer_i\mat{Z}_i\indict\left\{\mat{Z}_i^\top\vect{w}\right\}\right) \nonumber \\
 = &\left(\sum_{i=1}^{k}\secondlayer_i^2\mat{Z}_i\mat{Z}_i^\top\indict\left\{\mat{Z}_i^\top \vect{w}\ge 0\right\}+\sum_{i\neq j}\secondlayer_i\secondlayer_j\mat{Z}_i\mat{Z}_j^\top\indict\left\{\mat{Z}_i^\top\vect{w}\ge 0,\mat{Z}_j^\top\vect{w} \ge 0 \right\}\right)\vect{w} \label{eqn:w_grad_w_term}\\
- &\left(\sum_{i=1}^{k}\secondlayer_i\secondlayer_i^*\mat{Z}_i\mat{Z}_i^\top
\indict\left\{\mat{Z}_i^\top\vect{w}\ge 0,\mat{Z}_i^\top\vect{w}^*\ge 0\right\}
+ \sum_{i\neq j}\secondlayer_i\secondlayer_j^*\mat{Z}_i\mat{Z}_j^*\indict\left\{\mat{Z}_i^\top\vect{w}\ge 0,\mat{Z}_j^\top \vect{w}^* \ge 0\right\}\right)\vect{w}^*. \label{eqn:w_grad_w*_term}
\end{align}
Now we calculate expectation of Equation~\eqref{eqn:w_grad_w_term} and~\eqref{eqn:w_grad_w*_term} separately.
For~\eqref{eqn:w_grad_w_term}, by first two formulas of Lemma~\ref{lem:gaussian_basic_facts}, we have \begin{align*}
&\left(\sum_{i=1}^{k}\secondlayer_i^2\mat{Z}_i\mat{Z}_i^\top\indict\left\{\mat{Z}_i^\top \vect{w}\ge 0\right\}+\sum_{i\neq j}\secondlayer_i\secondlayer_j\mat{Z}_i\mat{Z}_j^\top\indict\left\{\mat{Z}_i^\top\vect{w}\ge 0,\mat{Z}_j^\top\vect{w} \ge 0 \right\}\right)\vect{w} \\
& = \sum_{i=1}^{k}\secondlayer_i^2\cdot\frac{\vect{w}}{2} + \sum_{i\neq j}\secondlayer_i\secondlayer_j\frac{\vect{w}}{2\pi}.
\end{align*}
For~\eqref{eqn:w_grad_w*_term}, we use the second and third formula in Lemma~\ref{lem:gaussian_basic_facts} to obtain\begin{align*}
&\left(\sum_{i=1}^{k}\secondlayer_i\secondlayer_i^*\mat{Z}_i\mat{Z}_i^\top
\indict\left\{\mat{Z}_i^\top\vect{w}\ge 0,\mat{Z}_i^\top\vect{w}^*\ge 0\right\}
+ \sum_{i\neq j}\secondlayer_i\secondlayer_j^*\mat{Z}_i\mat{Z}_j^*\indict\left\{\mat{Z}_i^\top\vect{w}\ge 0,\mat{Z}_j^\top \vect{w}^* \ge 0\right\}\right)\vect{w}^* \\
= & \vect{\secondlayer}^\top\vect{\secondlayer}^*\left(\frac{1}{\pi}\left(\pi-\phi\right)\vect{w}^* + \frac{1}{\pi}\sin\phi\frac{\norm{\vect{w}^*}_2}{\norm{\vect{w}}_2}\vect{w}\right) + \sum_{i \neq j}\secondlayer_i\secondlayer_j^*\frac{1}{2\pi}\frac{\norm{\vect{w}^*}_2}{\norm{\vect{w}}_2}\vect{w}^.
\end{align*}
In summary, aggregating them together we have\begin{align*}
&\expect_{\mat{Z}}\left[\frac{\partial \ell\left(\mat{Z},\vect{w},\vect{\secondlayer}\right)}{\partial \vect{w}}\right] \\
= &\frac{1}{2\pi}\vect{\secondlayer}^\top \vect{\secondlayer}^*\left(\pi-\phi\right)\vect{w}^* + \left(
\frac{\norm{\vect{\secondlayer}}_2^2}{2} + \frac{\sum_{i\neq j}\secondlayer_i\secondlayer_j}{2\pi} + \frac{\vect{\secondlayer}^\top\vect{\secondlayer}^*\sin\phi}{2\pi}\frac{\norm{\vect{w}^*}_2}{\norm{\vect{w}}_2} + \frac{\sum_{i\neq j}\secondlayer_j\secondlayer_j^*}{2\pi}\frac{\norm{\vect{w}_*}_2}{\norm{\vect{w}}_2}
\right)\vect{w}.
\end{align*}
As a sanity check, this formula matches Equation (16) of~\citep{brutzkus2017globally} when $\vect{\secondlayer}=\vect{\secondlayer}^*=\vect{1}$.

Next, we calculate the expected gradient of $\vect{\secondlayer}$.
Recall the gradient formula of $\vect{\secondlayer}$\begin{align*}
\frac{\partial \ell(\mat{Z},\vect{\firstlayer},\vect{\secondlayer})}{\vect{\secondlayer}} = &\left(\vect{\secondlayer}^\top\relu{\mat{Z}\vect{\firstlayer}}-(\vect{\secondlayer}^*)^\top\relu{\mat{Z}\vect{\firstlayer}^*}\right)\relu{\mat{Z}\vect{\firstlayer}} \\
= & \relu{\mat{Z}\vect{\firstlayer}}\relu{\mat{Z}\vect{\firstlayer}}^\top\vect{\secondlayer} - \relu{\mat{Z}\vect{\firstlayer}}\relu{\mat{Z}\vect{\firstlayer}^*}^\top\vect{\secondlayer}^* 
\end{align*}
Taking expectation we have
\begin{align*}
\frac{\partial \ell\left(\vect{\firstlayer},\vect{\secondlayer}\right)}{\partial \vect{\secondlayer}} =  \mat{A}\left(\vect{\firstlayer}\right)\vect{\secondlayer} - \mat{B}\left(\vect{\firstlayer},\vect{\firstlayer}^*\right)\vect{\secondlayer}^*
\end{align*}
where $\mat{A}\left(\vect{\firstlayer}\right)$ and $\mat{B}\left(\vect{\firstlayer},\vect{\firstlayer}^*\right)$ are defined in Equation~\eqref{eqn:A} and~\eqref{eqn:B}.
Plugging in the formulas for $\mat{A}\left(\vect{\firstlayer}\right)$ and $\mat{B}\left(\vect{\firstlayer},\vect{\firstlayer}^*\right)$ derived in the proof of Theorem~\ref{thm:gaussian_input_obj_WN} we obtained the desired result.

\end{proof}

\begin{lem}[Useful Identities]\label{lem:gaussian_basic_facts}
Given $\vect{w}$, $\vect{w}^*$ with angle $\phi$ and $\vect{Z}$ is a Gaussian random vector, then
\begin{align*}
\expect\left[\vect{z}\vect{z}^\top\indict\left\{\vect{z}^\top\vect{w}\ge 0\right\}\right]\vect{w} =& \frac{1}{2}\vect{w}\\
\expect\left[\vect{z}\indict\left\{\vect{z}^\top\vect{w}\ge 0\right\}\right] = & \frac{1}{\sqrt{2\pi}}\frac{\vect{w}}{\norm{\vect{w}}_2} \\
\expect\left[\vect{z}\vect{z}^\top\indict\left\{\vect{z}^\top\vect{w}\ge 0,\vect{z}^\top \vect{w}_*\ge 0\right\}\right]\vect{w}_* = &\frac{1}{2\pi}\left(\pi-\phi\right)\vect{w}^* + \frac{1}{2\pi}\sin\phi\frac{\norm{\vect{w}^*}_2}{\norm{\vect{w}}_2}\vect{w} \\
\expect\left[\relu{\vect{z}^\top\vect{w}}\relu{\vect{z}^\top\vect{w}_*}\right] = & \frac{1}{2\pi}\left(\cos\phi\left(\pi-\phi\right)+\sin\phi\right)\norm{\vect{w}}_2\norm{\vect{w}^*}_2
\end{align*}
\end{lem}
\begin{proof}
Consider an orthonormal basis of $\mathbb{R}^{d \times d}$: $\left\{\vect{e}_i\vect{e}_j^\top\right\}$ with $\vect{e}_1 \parallel \vect{w}$.
Then for $i \neq j$, we know \[\langle \vect{e}_i\vect{e}_j, \expect\left[\vect{z}\vect{z}^\top\indict\left\{\vect{z}^\top\vect{w}\ge 0\right\}\right] \rangle= 0\] by the independence properties of Gaussian random vector.
For $i=j=1$, \begin{align*}\langle \vect{e}_i\vect{e}_j^\top, \expect\left[\vect{z}\vect{z}^\top\indict\left\{\vect{z}^\top\vect{w}\ge 0\right\}\right]\rangle
=    \expect\left[\left(\vect{z}^\top\vect{w}\right)^2\indict\left\{\vect{z}^\top\vect{w}\ge 0 \right\}\right] =  \frac{1}{2}
\end{align*}
where the last step is by the property of half-Gaussian.
For $i=j\neq j$, $\langle \vect{e}_i\vect{e}_j^\top, \expect\left[\vect{z}\vect{z}^\top\indict\left\{\vect{z}^\top\vect{w}\ge 0\right\}\right]\rangle
=1$ by standard Gaussian second moment formula.
Therefore, $ \expect\left[\vect{z}\vect{z}^\top\indict\left\{\vect{z}^\top\vect{w}\ge 0\right\}\right]\vect{w} = \frac{1}{2}\vect{w}$.
$\expect\left[\vect{z}\indict\left\{\vect{z}^\top\vect{w}\ge 0\right\}\right] =  \frac{1}{\sqrt{2\pi}}\vect{w}$ can be proved by mean formula of half-normal distribution.
To prove the third identity, consider an orthonormal basis of $\mathbb{R}^{d \times d}$: $\left\{\vect{e}_i\vect{e}_j^\top \right\}$ with $\vect{e}_1 \parallel \vect{w}_*$ and $\vect{w}$ lies in the plane spanned by $\vect{e}_1$ and $\vect{e}_2$.
Using the polar representation of 2D Gaussian random variables ($r$ is the radius and $\theta$ is the angle with $\diff P_r = r\exp(-r^2/2)$ and $\diff P_\theta = \frac{1}{2\pi}$): \begin{align*}
\langle \vect{e}_1\vect{e}_1^\top, \expect\left[\vect{z}\vect{z}^\top\indict\left\{\vect{z}^\top\vect{w}\ge 0,\vect{z}^\top \vect{w}_*\ge 0\right\}\right] \rangle 
= & \frac{1}{2\pi}\int_{0}^{\infty} r^3\exp\left(-r^2/2\right) \diff r \cdot \int_{-\pi/2+\phi}^{\pi/2}\cos^2\theta \diff \theta  \\
= & \frac{1}{2\pi}\left(\pi-\phi +\sin\phi \cos\phi\right),\\
\langle \vect{e}_1\vect{e}_2^\top, \expect\left[\vect{z}\vect{z}^\top\indict\left\{\vect{z}^\top\vect{w}\ge 0,\vect{z}^\top \vect{w}_*\ge 0\right\}\right] \rangle 
= & \frac{1}{2\pi}\int_{0}^{\infty} r^3\exp\left(-r^2/2\right) \diff r \cdot \int_{-\pi/2+\phi}^{\pi/2}\sin\theta\cos\theta \diff \theta  \\
= & \frac{1}{2\pi}\left(\sin^2\phi\right),\\
\langle \vect{e}_2\vect{e}_2^\top, \expect\left[\vect{z}\vect{z}^\top\indict\left\{\vect{z}^\top\vect{w}\ge 0,\vect{z}^\top \vect{w}_*\ge 0\right\}\right] \rangle 
= & \frac{1}{2\pi}\int_{0}^{\infty} r^3\exp\left(-r^2/2\right) \diff r \cdot \int_{-\pi/2+\phi}^{\pi/2}\sin^2\theta \diff \theta  \\
= & \frac{1}{2\pi}\left(\pi - \phi -\sin\phi \cos\phi\right).
\end{align*}
Also note that $\vect{e}_2 = \frac{\bar{\vect{w}}-\cos \phi\vect{e}_1}{\sin\phi}$.
Therefore \begin{align*}
\expect\left[\vect{z}\vect{z}^\top\indict\left\{\vect{z}^\top\vect{w}\ge 0,\vect{z}^\top \vect{w}_*\ge 0\right\}\right]\vect{w}_* = &\frac{1}{2\pi}\left(\pi-\phi +\sin\phi \cos\phi\right)\vect{w}^* + \frac{1}{2\pi}\sin^2\phi\cdot \frac{\bar{\vect{w}}-\cos\phi\vect{e}_1}{\sin\phi}\norm{\vect{w}^*}_2 \\
= & \frac{1}{2\pi}\left(\pi-\phi\right)\vect{w}^* + \frac{1}{2\pi}\sin\phi\frac{\norm{\vect{w}^*}_2}{\norm{\vect{w}}_2}\vect{w}.
\end{align*}
For the fourth identity, focusing on the plane spanned by $\vect{w}$ and $\vect{w}_*$, using the polar decomposition, we have \begin{align*}
\expect\left[\relu{\vect{z}^\top\vect{w}}\relu{\vect{z}^\top\vect{w}_*}\right] &= \frac{1}{2\pi}\int_{0}^{\infty}r^3\exp\left(-r^2/2\right)\diff r\cdot \int_{-\pi/2+\phi}^{\pi/2} \left(\cos\theta\cos\phi+\sin\theta\sin\phi\right)\cos\theta \diff \theta \norm{\vect{w}}_2\norm{\vect{w}^*}_2\\
& = \frac{1}{2\pi}\left(\cos\phi\left(\pi-\phi+\sin\phi\cos\phi\right)+\sin^3\phi\right)\norm{\vect{w}}_2\norm{\vect{w}^*}_2.
\end{align*}
\end{proof}

\section{Proofs of Qualitative Convergence Results}
\label{sec:proof_qualitative}
\begin{proof}[Proof of Lemma~\ref{lem:stationary_point}]
When Algorithm~\ref{algo:weight_normalization_gd} converges, since $\vect{\secondlayer}^\top \vect{\secondlayer}^* \neq 0$ and $\norm{\vect{\firstlayerWN}}_2 < \infty$, using the gradient formula in Theorem~\ref{thm:expected_gradient_WN}, we know that either $\pi-\phi = 0$ or 
$\left(\mat{I} -\frac{\vect{\firstlayerWN}\vect{\firstlayerWN}^\top}{\norm{\vect{\firstlayerWN}}_2^2}\right)\vect{\firstlayer}^*=\vect{0}.$
For the second case, since $\mat{I} -\frac{\vect{\firstlayerWN}\vect{\firstlayerWN}^\top}{\norm{\vect{\firstlayerWN}}_2^2}$ is a projection matrix on the complement space of $\vect{\firstlayerWN}$, $\left(\mat{I} -\frac{\vect{\firstlayerWN}\vect{\firstlayerWN}^\top}{\norm{\vect{\firstlayerWN}}_2^2}\right)\vect{\firstlayer}^*=\vect{0}$ is equivalent to $\theta\left(\vect{\firstlayerWN},\vect{\firstlayer}^*\right)=0$.
Once the angle between $\vect{\firstlayerWN}$ and $\vect{\firstlayer}^*$ is fixed, using the gradient formula for $\vect{\secondlayer}$ we have the desired formulas for saddle points.
\end{proof}

\begin{proof}[Proof of Lemma~\ref{lem:w_angle_converge}]
By the gradient formula of $\vect{w}$, if $\vect{\secondlayer}^\top \vect{\secondlayer}^* > 0$, the gradient is of the form $c\left(\mat{I} -\frac{\vect{\firstlayerWN}\vect{\firstlayerWN}^\top}{\norm{\vect{\firstlayerWN}}_2^2}\right)\vect{w}^*$ where $c > 0$.
Thus because $\mat{I} -\frac{\vect{\firstlayerWN}\vect{\firstlayerWN}^\top}{\norm{\vect{\firstlayerWN}}_2^2}$ is the projection matrix onto the complement space of $\vect{\firstlayerWN}$, the gradient update always makes the angle smaller. 
\end{proof}

\section{Proofs of Quantitative Convergence Results}
\label{sec:proof_dynamics}
\subsection{Useful Technical Lemmas}
\label{sec:useful_lemmas}
We first prove the lemma about the convergence of $\phi^t$.
\begin{proof}[Proof of Lemma~\ref{lem:first_layer_convergence_one_iter}]
We consider the dynamics of $\sin^2 \phi^t$.\begin{align*}
	    &\sin^2\phi^{t+1}\\ 
	=  & 1 -\frac{\left(\left(\vect{\firstlayerWN}^{t+1}\right)^\top\vect{\firstlayer}^*\right)^2}{\norm{\vect{\firstlayerWN}^{t+1}}_2^2\norm{\vect{\firstlayer}^*}_2^2} \\
	= & 1- \frac{
		\left(\left(\vect{\firstlayerWN^t} - \eta \frac{\partial \ell}{\partial \vect{\firstlayerWN^t}}\right)^\top \vect{\firstlayer}^*\right)^2
		}{
		\left(\norm{\vect{\firstlayerWN}^t}_2^2+\eta^2\left(\frac{\partial \ell}{\partial \vect{\firstlayerWN^t}}\right)^2\right)\norm{\vect{\firstlayer}^*}_2^2
		}\\
	= & 1 - \frac{\left(
		\left(\vect{\firstlayerWN}^t\right)^\top \vect{\firstlayerWN} + \eta \frac{\left(\vect{\secondlayer}^t\right)^\top\vect{\secondlayer}^*\left(\pi-\phi^t\right)}{2\pi\norm{\vect{\firstlayerWN}}_2}\cdot \sin^2\phi^t\norm{\vect{\firstlayer}}_2^2
		\right)^2}{
		\norm{\vect{\firstlayerWN}^t}_2^2\norm{\vect{\firstlayer}^*}_2^2 + \eta^2\left(\frac{\left(\vect{\secondlayer}^t\right)^\top \vect{\secondlayer}^*\left(\pi-\phi^t\right)}{2\pi}\right)^2\frac{\sin^2\phi^t\norm{\vect{\firstlayer}^*}_2^4}{\norm{\vect{\firstlayerWN}^t}_2^2}
		}\\
	\le & 1 - \frac{
		\norm{\vect{\firstlayerWN}^t}_2^2\norm{\vect{\firstlayer}^*}_2^2\cos^2\phi^t + 2\eta \norm{\vect{\firstlayer}^*}_2^3\cdot\frac{\left(\vect{\secondlayer}^t\right)^\top\vect{\secondlayer}\left(\pi-\phi\right)}{2\pi}\cdot \sin^2\phi^t\cos \phi^t
		}{\norm{\vect{\firstlayerWN}^t}_2^2\norm{\vect{\firstlayer}^*}_2^2 + \eta^2\left(\frac{\left(\vect{\secondlayer}^t\right)^\top \vect{\secondlayer}^*\left(\pi-\phi^t\right)}{2\pi}\right)^2\frac{\sin^2\phi^t\norm{\vect{\firstlayer}^*}_2^4}{\norm{\vect{\firstlayerWN}^t}_2^2}} \\
	= & \frac{
		\sin^2\phi^t - 2\eta\frac{\norm{\vect{\firstlayer}^*}_2}{\norm{\vect{\firstlayerWN}^t}_2^2}\cdot\frac{\left(\vect{\secondlayer}^t\right)^\top\vect{\secondlayer}\left(\pi-\phi\right)}{2\pi}\cdot \sin^2\phi^t\cos \phi^t + \eta^2\left(\frac{\left(\vect{\secondlayer}^t\right)^\top\vect{\secondlayer}^*\left(\pi-\phi\right)}{2\pi}\right)^2\sin^2\phi^t\left(\frac{\norm{\vect{\firstlayer}^*}_2}{\norm{\vect{\firstlayerWN}}_2^2}\right)^2
		}{
		1 + \eta^2\left(\frac{\left(\vect{\secondlayer}^t\right)^\top\vect{\secondlayer}^*\left(\pi-\phi\right)}{2\pi}\right)^2\sin^2\phi^t\left(\frac{\norm{\vect{\firstlayer}^*}_2}{\norm{\vect{\firstlayerWN}^t}_2^2}\right)^2
		} \\
	\le & \sin^2\phi^t -
	2\eta\frac{\norm{\vect{\firstlayer}^*}_2}{\norm{\vect{\firstlayerWN}^t}_2^2}\cdot\frac{\left(\vect{\secondlayer}^t\right)^\top\vect{\secondlayer}\left(\pi-\phi\right)}{2\pi}\cdot \sin^2\phi^t\cos \phi^t + \eta^2\left(\frac{\left(\vect{\secondlayer}^t\right)^\top\vect{\secondlayer}^*\left(\pi-\phi\right)}{2\pi}\right)^2\sin^2\phi^t\left(\frac{\norm{\vect{\firstlayer}^*}_2}{\norm{\vect{\firstlayerWN}^t}_2^2}\right)^2
\end{align*}
where in the first inequality we dropped term proportional to $O(\eta^4)$ because it is negative, in the last equality, we divided numerator and denominator by $\norm{\vect{\firstlayerWN}^t}_2^2\norm{\vect{\firstlayer}^*}_2^2$ and the last inequality we dropped the denominator because it is bigger than $1$.
Therefore, recall $\lambda^t = \frac{\norm{\vect{\firstlayer}^*}_2\left(\left(\vect{\secondlayer}^t\right)^\top\vect{\secondlayer}^*\right)\left(\pi-\phi^t\right)}{2\pi\norm{\vect{\firstlayerWN}^t}_2^2}$ and we have \begin{align} 
\sin^2\phi^{t+1} \le \left(1-2\eta \cos\phi^t\lambda^t + \eta^2\left(\lambda^t\right)^2\right)\sin^2\phi^t. \label{eqn:sin_phi_square_inequality}
\end{align}
To this end, we need to make sure $\eta \le \frac{\cos\phi^t}{\lambda^t}$.
Note that since $\norm{\vect{\firstlayerWN}^t}_2^2$ is monotonically increasing, it is lower bounded by $1$.
Next notice $\phi^t \le \pi/2$.
Finally, from Lemma~\ref{lem:upper_bound_second_layer}, we know $\left(\vect{\secondlayer}^t\right)^\top\vect{\secondlayer}^* \le \left(\norm{\vect{\secondlayer}^*}_2^2 +\left(\vect{1}^\top\vect{\secondlayer}^*\right)^2\right)\norm{\vect{\firstlayer}}_2^2$. 
Combining these, we have an upper bound \[\lambda^t \le \frac{\left(\norm{\vect{\secondlayer}^*}_2^2+\left(\vect{1}^\top\vect{\secondlayer}^*\right)^2\right)\norm{\vect{\firstlayer}^*}_2^2}{4}.\]
Plugging this back to Equation~\eqref{eqn:sin_phi_square_inequality} and use our assumption on $\eta$, we have \begin{align*}
\sin^2\phi^{t+1}\le \left(1-\eta\cos\phi^t\lambda^t\right)\sin^2\phi^t.
\end{align*}
\end{proof}

\begin{lem}\label{lem:lower_bound_second_layer}
$\left(\vect{\secondlayer}^{t+1}\right)^\top \vect{\secondlayer}^* \ge \min\left\{
\left(\vect{\secondlayer}^{t}\right)^\top \vect{\secondlayer}^* + \eta\left(\frac{g(\phi^t)-1}{\pi-1}\norm{\vect{\secondlayer}^*}_2^2 - \left(\vect{\secondlayer}^{t}\right)^\top \vect{\secondlayer}^*\right), \frac{g(\phi^t)-1}{\pi-1}\norm{\vect{\secondlayer}^*}_2^2
\right\}$
\end{lem}
\begin{proof}
Recall the dynamics of $\left(\vect{\secondlayer}^{t}\right)^\top \vect{\secondlayer}^*$.
\begin{align*}
\left(\vect{\secondlayer}^{t+1}\right)^\top \vect{\secondlayer}^*  = & \left(1-\frac{\eta\left(\pi-1\right)}{2\pi}\right)\left(\vect{\secondlayer}^t\right)^\top \vect{\secondlayer}^* + \frac{\eta\left(g(\phi^t)-1\right)}{2\pi}\norm{\vect{\secondlayer}^*}_2^2 + \frac{\eta}{2\pi}\left(\left(\vect{1}^\top\vect{\secondlayer}^*\right)^2-\left(\vect{1}^\top\vect{\secondlayer}^*\right)\left(\vect{1}^\top \vect{\secondlayer}^t\right)\right) \\
\ge &  \left(1-\frac{\eta\left(\pi-1\right)}{2\pi}\right)\left(\vect{\secondlayer}^t\right)^\top \vect{\secondlayer}^* + \frac{\eta\left(g(\phi^t)-1\right)}{2\pi}\norm{\vect{\secondlayer}^*}_2^2
\end{align*}
where the inequality is due to Lemma~\ref{lem:sum_beta_converge}.
If $\left(\vect{\secondlayer}^t\right)^\top \vect{\secondlayer}^* \ge \frac{g(\phi^t)-1}{\pi-1}\norm{\vect{\secondlayer}^*}_2^2$, \begin{align*}
\left(\vect{\secondlayer}^{t+1}\right)^\top \vect{\secondlayer}^*  \ge & \left(1-\frac{\eta\left(\pi-1\right)}{2\pi}\right)\frac{g(\pi^t)-1}{\pi-1}\norm{\vect{\secondlayer}^*}_2^2 + \frac{\eta\left(g(\phi^t)\right)}{\pi-1}\norm{\vect{\secondlayer}^*}_2^2 \\
= &\frac{g(\phi^t)-1}{\pi-1}\norm{\vect{\secondlayer}^*}_2^2.
\end{align*}
If $\left(\vect{\secondlayer}^t\right)^\top \vect{\secondlayer}^* \le \frac{g(\phi^t)-1}{\pi-1}\norm{\vect{\secondlayer}^*}_2^2$, simple algebra shows $\left(\vect{\secondlayer}^{t+1}\right)^\top \vect{\secondlayer}^*$ increases by at least \[\eta\left(\frac{g(\phi^t)-1}{\pi-1}\norm{\vect{\secondlayer}^*}_2^2 - \left(\vect{\secondlayer}^{t}\right)^\top \vect{\secondlayer}^*\right).\]

\end{proof}
A simple corollary is $\vect{\secondlayer}^\top \vect{\secondlayer}^*$ is uniformly lower bounded.
\begin{cor}
For all $t=1,2,\ldots$, $\left(\vect{\secondlayer}^t\right)^\top \vect{\secondlayer}^* \ge \min\left\{\left(\vect{\secondlayer}^0\right)^\top \vect{\secondlayer}^*, \frac{g(\phi^0)-1}{\pi-1}\norm{\vect{\secondlayer}^*}_2^2\right\}$.
\end{cor}

This lemma also gives an upper bound of number of iterations to make $\vect{\secondlayer}^\top\vect{\secondlayer}^* = \Theta\left(\norm{\vect{\secondlayer}^*}_2^2\right)$.
\begin{cor}\label{cor:inner_prduct_convergences}
If $g(\phi) - 1 = \Omega\left(1\right)$, then after $\frac{1}{\eta}$ iterations, $\vect{\secondlayer}^\top\vect{\secondlayer}^* = \Theta\left(\norm{\vect{\secondlayer}^*}_2^2\right)$.
\end{cor}
\begin{proof}
Note if $g(\phi) - 1 = \Omega\left(1\right)$ and $\vect{\secondlayer}^\top \vect{\secondlayer}^* \le \frac{1}{2}\cdot\frac{g(\phi)}{\pi-1}\norm{\vect{\secondlayer}^*}_2^2$, each iteration  $\vect{\secondlayer}^\top \vect{\secondlayer}^*$ increases by $\eta\frac{g(\phi)}{\pi-1}\norm{\vect{\secondlayer}^*}_2^2$. 
\end{proof}

We also need an upper bound of $\left(\vect{\secondlayer}^t\right)^\top\vect{\secondlayer}^*$.
\begin{lem}\label{lem:upper_bound_second_layer}
For $t=0,1,\ldots$, $\left(\vect{\secondlayer}^t\right)^\top \vect{\secondlayer}^* \le \left(\norm{\vect{\secondlayer}^*}_2^2+\left(\vect{1}^\top\vect{\secondlayer}^*\right)^2\right)\norm{\vect{\firstlayer}^*}_2^2$.
\end{lem}
\begin{proof}
Without loss of generality, assume $\norm{\vect{\firstlayer}^*}_2=1$.
Again, recall the dynamics of $\left(\vect{\secondlayer}^t\right)^\top \vect{\secondlayer}^*$.
\begin{align*}
\left(\vect{\secondlayer}^{t+1}\right)^\top \vect{\secondlayer}^*  = & \left(1-\frac{\eta\left(\pi-1\right)}{2\pi}\right)\left(\vect{\secondlayer}^t\right)^\top \vect{\secondlayer}^* + \frac{\eta\left(g(\phi^t)-1\right)}{2\pi}\norm{\vect{\secondlayer}^*}_2^2 + \frac{\eta}{2\pi}\left(\left(\vect{1}^\top\vect{\secondlayer}^*\right)^2-\left(\vect{1}^\top\vect{\secondlayer}^*\right)\left(\vect{1}^\top \vect{\secondlayer}^t\right)\right) \\
\le & \left(1-\frac{\eta\left(\pi-1\right)}{2\pi}\right)\left(\vect{\secondlayer}^t\right)^\top\vect{\secondlayer}^* + \frac{\eta\left(\pi-1\right)}{2\pi}\norm{\vect{\secondlayer}^*}_2^2 + \frac{\eta\left(\pi-1\right)}{2\pi}\left(\vect{1}^\top\vect{\secondlayer}^*\right)^2.
\end{align*}
Now we prove by induction, suppose the conclusion holds at iteration $t$, $\left(\vect{\secondlayer}^t\right)^\top \vect{\secondlayer}^* \le \norm{\vect{\secondlayer}^*}_2^2+\left(\vect{1}^\top\vect{\secondlayer}^*\right)^2$.
Plugging in we have the desired result.
\end{proof}

\subsection{Convergence of Phase I}
\label{sec:proof_phase_1}
In this section we prove the convergence of Phase I.
\begin{proof}[Proof of Convergence of Phase I]
Lemma~\ref{lem:firstlayerWN_norm_always_small} implies after $O\left(\frac{1}{\cos \phi^0 \beta^0}\right)$ iterations, $\cos \phi^t = \Omega\left(1\right)$, which implies $\frac{g(\phi^t)-1}{\pi-1} = \Omega\left(1\right)$.
Using Corollary~\ref{cor:inner_prduct_convergences}, we know after $O\left(\frac{1}{\eta}\right)$ iterations we have $\left(\vect{\secondlayer}^t\right)^\top \vect{\secondlayer}^*\norm{\vect{\firstlayer}^*} = \Omega\left(\norm{\vect{\firstlayer}^*}_2^2\norm{\vect{\secondlayer}^*}_2^2\right)$.
\end{proof}
The main ingredient of the proof of phase I is the follow lemma where we use a joint induction argument to show the convergence of $\phi^t$ and a uniform upper bound of $\norm{\vect{\firstlayerWN}^t}_2$.
\begin{lem}
\label{lem:firstlayerWN_norm_always_small}
	Let $\beta^0 = \min\left\{\left(\vect{\secondlayer}^0\right)^\top \vect{\secondlayer}^*,\left(g(\phi^0)-1\right)\norm{\vect{\secondlayer}^*}_2^2\right\}\norm{\vect{\firstlayer}^*}_2^2$.
	If the step size satisfies $\eta \le \min\left\{\frac{\beta^*\cos\phi^0}{8\left(\norm{\vect{\secondlayer}^*}_2^2+\left(\vect{1}^\top \vect{\secondlayer}^*\right)^2\right)\norm{\vect{\firstlayer}^*}_2^2}, \frac{\cos\phi^0}{\left(\norm{\vect{\secondlayer}^*}_2^2+\left(\vect{1}^\top \vect{\secondlayer}^*\right)^2\right)\norm{\vect{\firstlayer}^*}_2^2}, \frac{2\pi}{k+\pi-1}\right\}$, we have for $t=0,1,\ldots$\begin{align*}
	\sin^2\phi^{t} \le \left(1-\eta\cdot\frac{\cos\phi^0\beta^0}{8}\right)^t \text{ and }\norm{\vect{\firstlayerWN}^t}_2 \le 2.
	\end{align*} 
\end{lem}
\begin{proof}
We prove by induction.
The initialization ensure when $t=0$, the conclusion is correct.
Now we consider the dynamics of $\norm{\vect{\firstlayerWN}^t}_2^2$. 
Note because the gradient of $\vect{\firstlayerWN}$ is orthogonal to $\vect{\firstlayerWN}$~\citep{salimans2016weight}, we have a simple dynamic of $\norm{\vect{\firstlayerWN}^t}_2^2$.
\begin{align*}
\norm{\vect{\firstlayerWN}^{t}}_2^2 = &\norm{\vect{\firstlayerWN}^{t-1}}_2^2 + \eta^2\norm{\frac{\partial \ell\left(\vect{\firstlayerWN},\vect{\secondlayer}\right)}{\partial \vect{\firstlayerWN}}}_2^2 \\
= & \norm{\vect{\firstlayerWN}^{t-1}}_2^2 + \eta^2\left(\frac{\left(\vect{\secondlayer}^t\right)^\top\vect{\secondlayer}^*\left(\pi-\phi^{t-1}\right)}{2\pi}\right)^2\frac{\sin^2\phi^t\norm{\vect{\firstlayer}^*}_2^2}{\norm{\vect{\firstlayerWN}^t}_2^2}\\
\le & \norm{\vect{\firstlayerWN}^{t-1}}_2^2 + \eta^2\left(\norm{\vect{\secondlayer}^*}_2^2+\left(\vect{1}^\top \vect{\secondlayer}^*\right)^2\right)\norm{\vect{\firstlayer}^*}_2^2\sin^2\phi^{t-1}\\
= & 1 + \eta^2\left(\norm{\vect{\secondlayer}^*}_2^2+\left(\vect{1}^\top \vect{\secondlayer}^*\right)^2\right)\norm{\vect{\firstlayer}^*}_2^2\sum_{i=1}^{t-1}\sin^2\phi^i\\
\le & 1 +\eta^2\left(\norm{\vect{\secondlayer}^*}_2^2+\left(\vect{1}^\top \vect{\secondlayer}^*\right)^2\right)\norm{\vect{\firstlayer}^*}_2^2\frac{8}{\eta\cos\phi^0\beta^0}\\
\le & 2
\end{align*}
where the first inequality is by Lemma~\ref{lem:upper_bound_second_layer} and the second inequality we use our induction hypothesis.
Recall $\lambda^t = \frac{\norm{\vect{\firstlayer}^*}_2\left(\left(\vect{\secondlayer}^t\right)^\top\vect{\secondlayer}^*\right)\left(\pi-\phi^t\right)}{2\pi\norm{\vect{\firstlayerWN}^t}_2^2}$.
The uniform upper bound of $\norm{\vect{\firstlayerWN}}_2$ and the fact that $\phi^t \le \pi/2$ imply a lower bound $\lambda^t \ge \frac{\beta^0}{8}$.
Plugging in Lemma~\ref{lem:first_layer_convergence_one_iter}, we have \begin{align*}
\sin^2\phi^{t+1} \le \left(1-\eta \frac{\cos\phi^0\beta^0}{8}\right)\sin^2\phi^t \le \left(1-\eta \frac{\cos\phi^0\beta^0}{8}\right)^{t+1}.
\end{align*}
We finish our joint induction proof.
\end{proof}

\subsection{Analysis of Phase II}
\label{sec:analaysis_phase_2}
In this section we prove the convergence of phase II and necessary auxiliary lemmas.
\begin{proof}[Proof of Convergence of Phase II]
At the beginning of Phase II,  $\left(\vect{\secondlayer}^{T_1}\right)^\top \vect{\secondlayer}^*\norm{\vect{\firstlayer}^*} = \Omega\left(\norm{\vect{\firstlayer}^*}_2^2\norm{\vect{\secondlayer}^*}_2^2\right)$ and $g(\phi^{T_1}) - 1 = \Omega\left(1\right)$.
Therefore, Lemma~\ref{lem:lower_bound_second_layer} implies for all $t=T_1,T_1+1,\ldots$, $\left(\vect{\secondlayer}^{t}\right)^\top \vect{\secondlayer}^*\norm{\vect{\firstlayer}^*} = \Omega\left(\norm{\vect{\firstlayer}^*}_2^2\norm{\vect{\secondlayer}^*}_2^2\right)$.
Combining with the fact that $\norm{\vect{\firstlayerWN}}_2 \le 2$ (c.f. Lemma~\ref{lem:firstlayerWN_norm_always_small}), we obtain a lower bound $\lambda_t \ge \Omega\left(\norm{\vect{\firstlayer}^*}_2^2\norm{\vect{\secondlayer}^*}_2^2\right)$
We also know that 
$\cos \phi^{T_1} = \Omega\left(1\right)$ and $\cos\phi^t$ is monotinically increasing (c.f. Lemma~\ref{lem:w_angle_converge}), so for all $t=T_1,T_1+1,\ldots$, $\cos\phi^t = \Omega\left(1\right)$.
Plugging in these two lower bounds into Theorem~\ref{lem:first_layer_convergence_one_iter}, we have \begin{align*}
\sin^2\phi^{t+1}  \le \left(1-\eta C\norm{\vect{\firstlayer}^*}_2^2\norm{\vect{\secondlayer}^*}_2^2\right)\sin^2\phi^t.
\end{align*} for some absolute constant $C$.
Thus, after $O\left(\frac{1}{\eta\norm{\vect{\firstlayer}^*}_2^2\norm{\vect{\secondlayer}^*}_2^2}\log\left(\frac{1}{\epsilon}\right)\right)$ iterations, we have $\sin^2\phi^t \le \min\left\{\epsilon^{10},\left(\epsilon\frac{\norm{\vect{\secondlayer}^*}_2}{\abs{\vect{1}^\top \vect{\secondlayer}^*}}\right)^{10}\right\}$, which implies $\pi- g(\phi^t) \le \min\left\{\epsilon,\epsilon\frac{\norm{\vect{\secondlayer}^*}_2}{\abs{\vect{1}^\top \vect{\secondlayer}^*}} \right\}$.
Now using Lemma~\ref{lem:second_layer_sum_convergence},Lemma~\ref{lem:second_layer_convergence} and Lemma~\ref{lem:prediction_error}, we have after $\widetilde{O}\left(
\frac{1}{\eta k}\log\left(\frac{1}{\epsilon}\right)
\right)$ iterations $\ell\left(\vect{\firstlayerWN},\vect{\secondlayer}\right) \le C_1\epsilon \norm{\vect{\secondlayer}^*}_2^2\norm{\vect{\firstlayer}^*}_2^2$ for some absolute constant $C_1$.
Rescaling $\epsilon$ properly we obtain the desired result.
\end{proof}

\subsubsection{Technical Lemmas for Analyzing Phase II}
\label{lem:}
In this section we provide some technical lemmas for analyzing Phase II.
Because of the positive homogeneity property, without loss of generality, we assume $\norm{\vect{\firstlayer}^*}_2=1$.
\begin{lem}\label{lem:second_layer_sum_convergence}
If $\pi-g(\phi^0) \le \epsilon\frac{\norm{\vect{\secondlayer}^*}_2}{\abs{\vect{1}^\top\vect{\secondlayer}^*}}$, after $T = O\left(\frac{1}{\eta k}\log\left(\frac{\abs{\vect{1}^\top \vect{\secondlayer}^*-\vect{1}^\top\vect{\secondlayer}^0}}{\epsilon\norm{\vect{\secondlayer}^*}_2}\right)\right)$ iterations, $\abs{\vect{1}^\top \vect{\secondlayer}^*- \vect{1}^\top\vect{\secondlayer}^{T} } \le 2\epsilon \norm{\vect{\secondlayer}^*}_2$.
\end{lem}
\begin{proof}
Recall the dynamics of $\vect{1}^\top \vect{\secondlayer}^t$.\begin{align*}
\vect{1}^\top \vect{\secondlayer}^{t+1} = &\left(1-\frac{\eta\left(k+\pi-1\right)}{2\pi}\right)\vect{1}^\top \vect{\secondlayer}^t + \frac{\eta\left(k+g(\phi^t)-1\right)}{2\pi}\vect{1}^\top \vect{\secondlayer}^* \\
& = \left(1-\frac{\eta\left(k+\pi-1\right)}{2\pi}\right)\vect{1}^\top \vect{\secondlayer}^t + \frac{\eta\left(k+g(\phi^t)-1\right)}{2\pi}\vect{1}^\top \vect{\secondlayer}^*.
\end{align*}
Assume $\vect{1}^\top \vect{\secondlayer}^* > 0$ (the other case is similar).
By Lemma~\ref{lem:sum_beta_converge} we know $\vect{1}^\top \vect{\secondlayer}^{t} < \vect{1}^\top \vect{\secondlayer}^*$ for all $t$.
Consider \begin{align*}
\vect{1}^\top \vect{\secondlayer}^*  - \vect{1}^\top \vect{\secondlayer}^{t+1} = \left(1-\frac{\eta\left(k+\pi-1\right)}{2\pi}\right)\left(\vect{1}^\top \vect{\secondlayer}^* - \vect{1}^\top \vect{\secondlayer}^*\right) + \frac{\eta\left(\pi-g(\phi^t)\right)}{2\pi}\vect{1}^\top \vect{\secondlayer}^*.
\end{align*}
Therefore we have \begin{align*}
\vect{1}^\top \vect{\secondlayer}^*  - \vect{1}^\top \vect{\secondlayer}^{t+1} - \frac{\left(\pi-g\left(\phi^t\right)\right)\vect{1}^\top \vect{\secondlayer}^*}{k+\pi-1} = \left(1-\frac{\eta\left(k+\pi-1\right)}{2\pi}\right)\left(\vect{1}^\top \vect{\secondlayer}^* - \vect{1}^\top \vect{\secondlayer}^* - \frac{\left(\pi-g\left(\phi^t\right)\right)\vect{1}^\top \vect{\secondlayer}^*}{k+\pi-1}\right).
\end{align*}
After $T = O\left(\frac{1}{\eta k}\log\left(\frac{\abs{\vect{1}^\top \vect{\secondlayer}^*-\vect{1}^\top\vect{\secondlayer}^0}}{\epsilon\norm{\vect{\secondlayer}^*}_2}\right)\right)$ iterations, we have $
\vect{1}^\top \vect{\secondlayer}^*  - \vect{1}^\top \vect{\secondlayer}^{t} - \frac{\left(\pi-g\left(\phi^t\right)\right)\vect{1}^\top \vect{\secondlayer}^*}{k+\pi-1} \le \epsilon \norm{\vect{\secondlayer}^*}_2
$, which implies$
\vect{1}^\top \vect{\secondlayer}^*  - \vect{1}^\top \vect{\secondlayer}^{t} \le 2\epsilon\norm{\vect{\secondlayer}^*}_2.
$
\end{proof}

\begin{lem}\label{lem:second_layer_convergence}
If $\pi-g(\phi^0) \le \epsilon\frac{\norm{\vect{\secondlayer}^*}_2}{\abs{\vect{1}^\top\vect{\secondlayer}^*}}$ and $\abs{\vect{1}^\top \vect{\secondlayer}^*- \vect{1}^\top\vect{\secondlayer}^{0} } \le \frac{\epsilon}{k} \norm{\vect{\secondlayer}^*}_2$, then after $T = O\left(\frac{1}{\eta }\log\left(\frac{\norm{ \vect{\secondlayer}^*-\vect{\secondlayer}^0}_2}{\epsilon \norm{\vect{\secondlayer}^*}_2}\right)\right)$ iterations, 
$\norm{ \vect{\secondlayer}^*-\vect{\secondlayer}^0}_2 \le C\epsilon \norm{\vect{\secondlayer}^*}_2$ for some absolute constant $C$.
\end{lem}
\begin{proof}
We first consider the inner product\begin{align*}
&\langle \frac{\partial \ell\left(\vect{\firstlayerWN}^t, \vect{\secondlayer}^t\right)}{\vect{\secondlayer}^t}, \vect{\secondlayer}^t - \vect{\secondlayer}^*\rangle \\
& = \frac{\pi-1}{2\pi}\norm{\vect{\secondlayer}^t-\vect{\secondlayer}^*}_2^2 - \frac{g(\phi^t)-\pi}{2\pi}\left(\vect{\secondlayer}^*\right)^\top\left(\vect{\secondlayer}^t-\vect{\secondlayer}^*\right) + \left(\vect{\secondlayer}^t-\vect{\secondlayer}^*\right)\vect{1}\vect{1}^\top\left(\vect{\secondlayer}^\top - \vect{\secondlayer}^*\right) \\
& \ge \frac{\pi-1}{2\pi}\norm{\vect{\secondlayer}^t-\vect{\secondlayer}^*}_2^2 - \frac{g(\phi^t)-\pi}{2\pi}\norm{\vect{\secondlayer}^*}_2\norm{\vect{\secondlayer}^t- \vect{\secondlayer}^*}_2.
\end{align*}
Next we consider the squared norm of gradient \begin{align*}
\norm{\frac{\partial \ell\left(\vect{\firstlayerWN},\vect{\secondlayer}\right)}{\partial \vect{\secondlayer}}}_2^2  = &\frac{1}{4\pi^2}\norm{\left(\pi-1\right)\left(\vect{\secondlayer}^t-\vect{\secondlayer}^*\right) + \left(\pi-g(\phi^t)\right)\vect{\secondlayer}^* + \vect{1}\vect{1}^\top\left(\vect{\secondlayer}^t-\vect{\secondlayer}^*\right)}_2^2 \\
\le &\frac{3}{4\pi^2}\left(\left(\pi-1\right)^2\norm{\vect{\secondlayer}^t-\vect{\secondlayer}^*}_2^2 +
 \left(\pi-g(\phi^t)\right)^2\norm{\vect{\secondlayer}^*}_2^2+
k^2\left(\vect{1}^\top\vect{\secondlayer}^t-\vect{1}^\top \vect{\secondlayer}^*\right)^2
\right).
\end{align*}
Suppose $\norm{\vect{\secondlayer}^t-\vect{\secondlayer}^*}_2 \le \epsilon\norm{\vect{\secondlayer}^*}_2$, then\begin{align*}
\langle \frac{\partial \ell\left(\vect{\firstlayerWN}^t, \vect{\secondlayer}^t\right)}{\vect{\secondlayer}^t}, \vect{\secondlayer}^t - \vect{\secondlayer}^*\rangle 
&\ge \frac{\pi-1}{2\pi}\norm{\vect{\secondlayer}^t-\vect{\secondlayer}^*}_2^2 - \frac{\epsilon^2}{2\pi}\norm{\vect{\secondlayer}^*}_2^2 \\
\norm{\frac{\partial \ell\left(\vect{\firstlayerWN},\vect{\secondlayer}\right)}{\partial \vect{\secondlayer}}}_2^2 
&\le 3\epsilon^2\norm{\vect{\secondlayer}^*}_2^2.
\end{align*}
Therefore we have \begin{align*}
&\norm{\vect{\secondlayer}^{t+1}-\vect{\secondlayer}^*}_2^2 \le\left(1-\frac{\eta\left(\pi-1\right)}{2\pi}\right)\norm{\vect{\secondlayer}^t-\vect{\secondlayer}^*}_2^2 +4\eta\epsilon^2\norm{\vect{\secondlayer}}^2\\
\Rightarrow & \norm{\vect{\secondlayer}^{t+1}-\vect{\secondlayer}^*}_2^2 - \frac{8\left(\pi-1\right)\epsilon^2\norm{\vect{\secondlayer}^*}_2^2}{\pi-1} \le \left(1-\frac{\eta\left(\pi-1\right)}{2\pi}\right)\left(\norm{\vect{\secondlayer}^{t}-\vect{\secondlayer}^*}_2^2 - \frac{8\left(\pi-1\right)\epsilon^2\norm{\vect{\secondlayer}^*}_2^2}{\pi-1}\right).
\end{align*}
Thus after $O\left(\frac{1}{\eta}\left(\frac{1}{\epsilon}\right)\right)$ iterations, we must have $\norm{\vect{\secondlayer}^{t+1}-\vect{\secondlayer}^*}_2^2 \le C\epsilon \norm{\vect{\secondlayer}^*}_2$ for some large absolute constant $C$.
Rescaling $\epsilon$, we obtain the desired result.
\end{proof}

\begin{lem}\label{lem:prediction_error}
If $\pi - g(\phi) \le \epsilon$ and $\norm{\vect{\secondlayer} - \vect{\secondlayer}^*\norm{\vect{\firstlayer}^*}_2} \le \epsilon \norm{\vect{\secondlayer}^*}_2\norm{\vect{\firstlayer}^*}_2$, then the population loss satisfies $\ell\left(\vect{\firstlayerWN},\vect{\secondlayer}\right) \le C\epsilon \norm{\vect{\secondlayer}^*}_2^2\norm{\vect{\firstlayer}^*}_2^2$ for some constant $C > 0$.
\end{lem}
\begin{proof}
The result follows by plugging in the assumptions in Theorem~\ref{thm:gaussian_input_obj_WN}.
\end{proof}

\section{Proofs of Initialization Scheme}
\label{sec:proof_init}
\begin{proof}[Proof of Theorem~\ref{thm:init}]
The proof of the first part of Theorem~\ref{thm:init} just uses the symmetry of unit sphere and ball and the second part is a direct application of Lemma 2.5 of~\citep{hardt2014noisy}.
Lastly, since $\vect{\secondlayer}^0 \sim \mathcal{B}\left(\vect{0}\frac{\abs{\vect{1}^\top \vect{\secondlayer}^*}}{\sqrt{k}}\right)$,  we have $\vect{1}^\top \vect{\secondlayer}^0 \le \norm{\vect{\secondlayer}^0}_1 \le  \sqrt{k}\norm{\vect{\secondlayer}^0}_2 \le \abs{\vect{1}^\top \vect{\secondlayer}^*}\norm{\vect{\firstlayer}^*}_2$ where the second inequality is due to H\"{o}lder's inequality.
\end{proof}

\section{Proofs of Converging to Spurious Local Minimum}
\label{sec:proof_local_min}
\begin{proof}[Proof of Theorem~\ref{thm:w_norm_1_gd_converge_bad}]
The main idea is similar to Theorem~\ref{thm:w_norm_1_gd_converge} but here we show $\vect{\firstlayer} \rightarrow -\vect{\firstlayer}^*$ (without loss of generality, we assume $\norm{\vect{\firstlayer}^*}_2 = 1$).
Different from Theorem~\ref{thm:w_norm_1_gd_converge}, here we need to prove the invariance $\vect{\secondlayer}^\top \vect{\secondlayer}^* < 0$, which implies our desired result.
We prove by induction, suppose  $\left(\vect{\secondlayer}^t\right)^\top \vect{\secondlayer}^* > 0$, 
$\abs{\vect{1}^\top\vect{\secondlayer}^t} \le \abs{\vect{1}^\top \vect{\secondlayer}^*}$, $g\left(\phi^0\right) \le  \frac{-2\left(\vect{1}^\top\vect{\secondlayer}\right)^2}{\norm{\vect{\secondlayer}^*}_2^2} + 1$ and $\eta < \frac{k+\pi-1}{2\pi}$.
Note $\abs{\vect{1}^\top\vect{\secondlayer}^t} \le \abs{\vect{1}^\top \vect{\secondlayer}^*}$ are satisfied by Lemma~\ref{lem:sum_beta_converge} and $g\left(\phi^0\right) \le  \frac{-2\left(\vect{1}^\top\vect{\secondlayer}\right)^2}{\norm{\vect{\secondlayer}^*}_2^2} + 1$ by our initialization condition and induction hypothesis that implies $\phi^t$ is increasing.
Recall the dynamics of $\left(\vect{\secondlayer}^t\right)^\top \vect{\secondlayer}^*$.
\begin{align*}
\left(\vect{\secondlayer}^{t+1}\right)^\top\vect{\secondlayer}^* = &\left(1-\frac{\eta\left(\pi-1\right)}{2\pi}\right)\left(\vect{\secondlayer}^t\right)^\top \vect{\secondlayer}^*  
+ \frac{\eta\left(g\left(\phi^t\right)-1\right)}{2\pi}\norm{\vect{\secondlayer}^*}_2^2 
+ \frac{\eta}{2\pi}\left(\left(\vect{1}^\top\vect{\secondlayer}^*\right)^2-\left(\vect{1}^\top\vect{\secondlayer}^t\right)\left(\vect{1}^\top\vect{\secondlayer}^*\right)\right) \\
\le & \frac{\eta\left(\left(g(\phi^t)-1\right)\norm{\vect{\secondlayer}^*}_2+2\left(\vect{1}^\top\vect{\secondlayer}^*\right)^2\right)}{2\pi} < 0
\end{align*}
where the first inequality we used our induction hypothesis on inner product between $\vect{\secondlayer}^t$ and $\vect{\secondlayer}^*$ and $\abs{\vect{1}^\top\vect{\secondlayer}^t} \le \abs{\vect{1}^\top \vect{\secondlayer}^*}$ and the second inequality is by induction hypothesis on $\phi^t$.
Thus when gradient descent algorithm converges, according Lemma~\ref{lem:stationary_point}, $\theta\left(\vect{\firstlayerWN},\vect{\firstlayer}^*\right) = \pi, \vect{\secondlayer} = \left(\vect{1}\vect{1}^\top + \left(\pi-1\right)\mat{I}\right)^{-1}\left(\vect{1}\vect{1}^\top - \mat{I}\right)\norm{\vect{\firstlayer}^*}_2\vect{\secondlayer}^*.$
Plugging these into Theorem~\ref{thm:gaussian_input_obj_WN}, with some routine algebra, we show $\ell\left(\vect{\firstlayerWN},\vect{\secondlayer}\right) = \Omega\left(\norm{\vect{\firstlayer}^*}_2^2\norm{\vect{\secondlayer}^*}_2^2\right)$.
\end{proof}

\end{document}